\documentclass[11pt]{article}
\usepackage[protrusion=false]{microtype}
\usepackage{xcolor}
\definecolor{darkblue}{rgb}{0.0, 0.0, 0.4}
\usepackage[colorlinks=true,allcolors=darkblue]{hyperref}
\usepackage{graphicx}
\usepackage{booktabs} 

\usepackage{amsmath}
\usepackage{amssymb}
\usepackage{amsthm,thmtools}
\usepackage[shortlabels]{enumitem}
\usepackage{amsfonts}
\RequirePackage{algorithm}
\RequirePackage{algorithmic}
\usepackage{bm,upgreek}
\usepackage{mathrsfs}  
\usepackage[round]{natbib}

\usepackage{calc}
\usepackage{subcaption}
\usepackage{authblk}

\theoremstyle{plain}
\newtheorem{theorem}{Theorem}

\newtheorem{condition}{Condition}
\newtheorem{lemma}[theorem]{Lemma}

\newcommand{\norm}[1]{\left\lVert#1\right\rVert}
\DeclareMathOperator*{\argmin}{argmin}

\DeclareMathOperator*{\esssup}{ess\,sup}
\DeclareMathOperator*{\essinf}{ess\,inf}
\usepackage[capitalize,noabbrev]{cleveref}

\DeclareFontFamily{U}{jkpmia}{}
\DeclareFontShape{U}{jkpmia}{m}{it}{<->s*jkpmia}{}
\DeclareFontShape{U}{jkpmia}{bx}{it}{<->s*jkpbmia}{}
\DeclareMathAlphabet{\mathfrak}{U}{jkpmia}{m}{it}
\SetMathAlphabet{\mathfrak}{bold}{U}{jkpmia}{bx}{it}

\usepackage[hang,flushmargin]{footmisc}
\newcommand\blfootnote[1]{
  \begingroup
  \renewcommand\thefootnote{}\footnote{#1}%
  \addtocounter{footnote}{-1}%
  \endgroup
}

\title{Causal isotonic calibration\\ for heterogeneous treatment effects}

\date{Version 2: June 5, 2023}

\author[1]{Lars van der Laan*}
\author[2]{Ernesto Ulloa-P\'erez*}
\author[3,1]{Marco Carone}
\author[1,3]{Alex Luedtke}
 
\affil[1]{\footnotesize Department of Statistics, University of Washington, USA}
\affil[2]{\footnotesize Department of Biostatistics, Epidemiology, and Informatics,  University of Pennsylvania, USA}
\affil[3]{\footnotesize Department of Biostatistics, University of Washington, USA}

\begin{document}

\allowdisplaybreaks
\maketitle
\blfootnote{* These authors contributed equally to this work.}

\begin{abstract}
We propose causal isotonic calibration, a novel nonparametric method for calibrating predictors of heterogeneous treatment effects. Furthermore, we introduce cross-calibration, a data-efficient variant of calibration that eliminates the need for hold-out calibration sets. Cross-calibration leverages cross-fitted predictors and generates a single calibrated predictor using all available data. Under weak conditions that do not assume monotonicity, we establish that both causal isotonic calibration and cross-calibration achieve fast doubly-robust calibration rates, as long as either the propensity score or outcome regression is estimated accurately in a suitable sense. The proposed causal isotonic calibrator can be wrapped around any black-box learning algorithm, providing robust and distribution-free calibration guarantees while preserving predictive performance.
\end{abstract}

\section{Introduction}

Estimation of causal effects via both randomized experiments and observational studies is critical to understanding the effects of interventions and informing policy. Moreover, it is often the case that understanding treatment effect heterogeneity can provide more insights than overall population effects \citep{obermeyer2016predicting, athey2017beyond}. For instance, a study of treatment effect heterogeneity can help elucidate the mechanism of an intervention, design policies targeted to subpopulations that can most benefit \citep{imbens2009recent}, and predict the effect of interventions in populations other than the ones in which they were developed. These necessities have arisen in a wide range of fields, such as marketing \citep{devriendt2018literature}, the social sciences \citep{imbens2009recent}, and the health sciences \citep{kent2018personalized}. For example, in the health sciences, heterogeneous treatment effects (HTEs) are of high importance to understanding and quantifying how certain exposures or interventions affect the health of various subpopulations \citep{dahabreh2016using, lee2020causal}. Potential applications include prioritizing treatment to certain sub-populations when treatment resources are scarce, or individualizing treatment assignments when the treatment can have no effect (or even be harmful) in certain subpopulations \citep{dahabreh2016using}. As an example, treatment assignment based on risk scores has been used to provide clinical guidance in cardiovascular disease prevention \citep{lloyd2019use} and to improve decision-making in oncology \citep{collins2015new,cucchiara2018genomic}.

A wide range of statistical methods are available for assessing HTEs, with recent examples including \citet{wager2018estimation}, \citet{carnegie2019examining}, \citet{lee2020causal}, \citet{yadlowsky2021evaluating}, and \citet{nie2021quasi}, among others. In particular, many methods, including  \citet{imbens2009recent} and \citet{dominici2020controlled}, scrutinize HTEs via conditional average treatment effects (CATEs). The CATE is the difference in the conditional mean of the counterfactual outcome corresponding to treatment versus control given covariates, which can be defined at a group or individual level. When interest lies in predicting treatment effect, the CATE can be viewed as the oracle predictor of the individual treatment effect (ITE) that can feasibly be learned from data. Optimal treatment rules have been derived based on the sign of the CATE estimator \citep{murphy2003optimal, robins2004optimal}, with more recent works incorporating the use of flexible CATE estimators \citep{luedtke2016statistical}. Thus, due to its wide applicability and scientific relevance, CATE estimation has been of great interest in statistics and data science.  

Regardless of its quality as a proxy for the true CATE, it is generally accepted that predictions from a given treatment effect predictor can still be useful for decision-making. However, theoretical guarantees for rational decision-making using a given treatment effect predictor typically hinge on the predictor being a good approximation of the true CATE. Accurate CATE estimation can be challenging because the nuisance parameters involved can be non-smooth, high-dimensional, or otherwise difficult to model correctly. Additionally, a CATE estimator obtained from samples of one population, regardless of its quality, may not generalize well to different target populations \citep{frangakis2009calibration}. Usually, CATE estimators (often referred to as learners) build upon estimators of the conditional mean outcome given covariates and treatment level (i.e., outcome regression), the probability of treatment given covariates (i.e., propensity score), or both. For instance, plug-in estimators such as those studied in \citet{kunzel2019metalearners} --- so-called T-learners --- are obtained by taking the difference between estimators of the outcome regression obtained separately for each treatment level. T-learners can suffer in performance because they rely on estimation of nuisance parameters that are at least as non-smooth or high-dimensional as the CATE, and are prone to the misspecification of involved outcome regression models; these issues can result in slow convergence or inconsistency of the CATE estimator. Doubly-robust and Neyman-orthogonal CATE estimation strategies like the DR-learner and R-learner \citep{wager2018estimation, foster2019orthogonal, nie2021quasi, kennedy2020optimal} mitigate some of these issues by allowing for comparatively fast CATE estimation rates even when nuisance parameters are estimated at slow rates. However, while less sensitive to the learning complexity of the nuisance parameters, their predictive accuracy in finite-samples still relies on potentially strong smoothness assumptions on the CATE. Even when the CATE is estimated consistently,
predictions based on statistical learning methods often produce biased predictions that overestimate or underestimate the true CATE in the extremes of the predicted values \citep{van2019models, dwivedi2020stable}. For example, the `pooled cohort equations' \citep{goff20142013} risk model used to predict cardiovascular disease has been found to underestimate risk in patients with lower socioeconomic status or chronic inflammatory diseases \citep{lloyd2019use}.  The implications of biased treatment effect predictors are profound when used to guide treatment decisions and can range from harmful use to withholding of treatment \citep{van2019calibration}.

Due to the consequence of treatment decision-making, it is essential to guarantee, under minimal assumptions, that treatment effect predictions are representative in magnitude and sign of the actual effects, even when the predictor is a poor approximation of the CATE. In prediction settings, the aim of bestowing these properties on a given predictor is commonly called \emph{calibration}. A calibrated treatment effect predictor has the property that the \emph{average} treatment effect among individuals with identical predictions is close to their shared prediction value. Such a predictor is more robust against the over-or-under estimation of the CATE in extremes of predicted values. It also has the property that the best predictor of the ITE given the predictor is the predictor itself, which facilitates transparent treatment decision-making. In particular, the optimal treatment rule \citep{murphy2003optimal} given only information provided by the predictor is the one that assigns the treatment predicted to be most beneficial. Consequently, the rule implied by a perfectly calibrated predictor is at least as favorable as the best possible static treatment rule that ignores HTEs. While complementing one another, the aims of calibration and prediction are fundamentally different. For instance, a constant treatment effect predictor can be well-calibrated even though it is a poor predictor of treatment effect heterogeneity \citep{gupta2020DistrFree}. In view of this, calibration methods are typically designed to be wrapped around a given black-box prediction pipeline to provide strong calibration guarantees while preserving predictive performance, thereby mitigating several prediction challenges mentioned previously.

 In the machine learning literature, calibration has been widely used to enhance prediction models for classification and regression \citep{bella2010calibration}. However, due to the comparatively little research on calibration of treatment effect predictors, such benefits have not been realized to the same extent in the context of heterogeneous treatment effect prediction. Several works have contributed to addressing this gap in the literature. \citet{brooks2012targeted} propose a targeted (or debiased) machine learning framework \citep{vanderLaanRose2011} for within-bins calibration that could be applied to the CATE setting. \citet{zhang2016new} and \citet{josey2022calibration} consider calibration of marginal treatment effect estimates for new populations but do not consider CATEs. \citet{dwivedi2020stable} consider estimating calibration error of CATE predictors for subgroup discovery using randomized experimental data. 
  \citet{chernozhukov2018generic} and \citet{leng2021calibration} propose CATE methods for linear calibration, a weaker form of calibration, in randomized experiments. For causal forests, \citet{athey2019estimating} evaluate model calibration using a doubly-robust estimator of the ATE among observations above or below the median predicted CATE. \citet{lei2021conformal} propose conformal inference methods for constructing calibrated prediction intervals for the ITE from a given predictor but do not consider calibration of the predictor itself. \citet{xu2022calibration} propose a nonparametric doubly-robust estimator of the calibration error of a given treatment effect predictor, which could be used to detect uncalibrated predictors. Our work builds upon the above works by providing a nonparametric doubly-robust method for calibrating treatment effect predictors in general settings. 

   This paper is organized as follows. In Section \ref{sec:secc2}, we introduce our notation and formally define calibration. There we also provide an overview of traditional calibration methods. In Section \ref{sec:secc4}, we outline our proposed approach, and we describe its theoretical properties in Section \ref{sec:secc5}. 
   In Section \ref{sec:secc6}, we examine the performance of our method in simulations.  
\section{Statistical Setup}\label{sec:secc2}

\subsection{Notation and Definitions} 
\label{sec::notation}
Suppose we observe $n$ independent and identically distributed realizations of data unit $O:=(W, A, Y)$ drawn from a distribution $P$, where $W \in \mathcal{W} \subset  \mathbb{R}^d$ is a vector of baseline covariates, $A \in \{0,1\}$ is a binary indicator of treatment, and $Y \in \mathcal{Y} \subset \mathbb{R}$ is an outcome. For instance, $W$ can include a patient's demographic characteristics and medical history, $A$ can indicate whether an individual is treated (1) or not (0), and $Y$ could be a binary indicator of a successful clinical outcome. We denote by $\mathcal{D}_n:= \{O_1,O_2,\ldots,O_n\}$ the observed dataset, with $O_i:=(W_i,A_i,Y_i)$ representing the observation on the $i^{th}$ study unit.

For covariate value $w\in\mathcal{W}$ and treatment level $a\in\{0,1\}$, we denote  by $\pi_0(w) := P(A = 1|W = w)$ the propensity score and by $\mu_0(a,w) := E(Y\,|\,A = a, W = w)$ the outcome regression. The individual treatment effect is $Y_1 - Y_0$, where $Y_a$ represents the potential outcome obtained by setting $A=a$. As convention, we take higher values of $Y_1-Y_0$ to be desirable. We assume that the contrast $\tau_0(w) := \mu_0(1,w) - \mu_0(0,w)$ equals the true CATE, $E(Y_1 - Y_0 \,|\, W = w)$, which holds under causal assumptions \citep{rubin1974}. Throughout, we denote by $\|\cdot\|$ the $L^2(P)$ norm, that is, $\|f\|^2=\int [f(w)]^2dP_W(w)$ for any given $P_W$-square integrable function $f:\mathcal{W}\rightarrow\mathbb{R}$, where $P_W$ is the marginal distribution of $W$ implied by $P$. We deliberately take as convention that the median $\text{median}\{x_1, x_2, \dots, x_k\}$ of a set $\{x_1, x_2, \dots, x_k\}$ equals the $\lfloor k/2\rfloor^{th}$ order statistic of this set, where $\lfloor k/2\rfloor := \max\{ z \in \mathbb{N}: z \leq k/2\}$.

Let $\tau:\mathcal{W} \rightarrow \mathbb{R}$ be a treatment effect predictor, that is, a function that maps a realization $w$ of $W$ to a treatment effect prediction $\tau(w)$. In practice, $\tau$ can be obtained using any black-box algorithm. Below, we first consider $\tau$ to be fixed, though we later address situations in which $\tau$ is learned from the data used for subsequent calibration. 
We define the calibration function $\gamma_0(\tau,w):= E[Y_1 - Y_0| \tau(W) = \tau(w)]$ as the conditional mean of the individual treatment effect given treatment effect score value $\tau(w)$. By the tower property, $\gamma_0(\tau,w)=E[\tau_0(W)\,|\,\tau(W) = \tau(w)]$, and so, expectations only involving $\gamma_0(\tau,W)$ and other functions of $W$ can be taken with respect to  $P_W$. 

 The solution to an isotonic regression problem is typically nonunique. Throughout this text, we follow \citet{groeneboom1993isotonic} in taking the unique c\`{a}dl\`{a}g piece-wise constant solution of the isotonic regression problem that can only take jumps at observed values of the predictor.
 
\subsection{Measuring Calibration and the Calibration-Distortion Decomposition}

Various definitions of risk predictor calibration have been proposed in the literature --- see \citet{gupta2021distribution} and \citet{gupta2020DistrFree} for a review. Here, we outline our definition of calibration and its rationale. Given a treatment effect predictor $\tau$, the best predictor of the individual treatment effect in terms of MSE is $w\mapsto\gamma_0(\tau,w) := E[Y_1 - Y_0\,|\, \tau(W) = \tau(w)]$. By the law of total expectation, this predictor has the property that, for any interval $[a,b)$,
\begin{equation}
\label{eq2:cal}
E\left\{\left[\tau_0(W)-\gamma_0(\tau,W)\right]I(\gamma_0(\tau,W)\in [a,b))\right\}=0\ .
\end{equation}
 Equation \ref{eq2:cal} indicates that $\gamma_0(\tau,\cdot)$ is perfectly calibrated on $[a,b)$. Therefore, when a given predictor $\tau$ is such that $\tau(W) = \gamma_0(\tau,W)$ with $P$-probability one, $\tau$ is said to be perfectly calibrated \citep{gupta2020DistrFree} for the CATE --- for brevity, we omit ``for the CATE'' hereafter when the type of calibration being referred to is clear from context. 

In general, perfect calibration cannot realistically be achieved in finite samples. A more modest goal is for the predictor $\tau$ to be approximately calibrated in that $\tau(w)$ is close to $\gamma_0(\tau,w)$ across all covariate values $w \in \mathcal{W}$. This naturally suggests the calibration measure:
\begin{equation}
\label{eq1:cal}
\text{CAL}(\tau):= \int\left[\gamma_0(\tau,w) - \tau(w)\right]^2dP_W(w).
\end{equation}
This measure, referred to as the $\ell_2$-expected calibration error, arises both in prediction \citep{gupta2020DistrFree} and in the assessment of treatment effect heterogeneity  \citep{xu2022calibration}. We note that $\text{CAL}(\tau)$ is zero if $\tau$ is perfectly calibrated. Additionally, averaging in $\text{CAL}(\tau)$ with respect to measures other than $P_W$ could be more relevant in certain applications; such cases can occur, for instance, when there is a change of population that results in covariate shift and we are interested in measuring how well $\tau$ is calibrated in the new population. 

Interestingly, the above calibration measure plays a role in a decomposition of the mean squared error (MSE) between the treatment predictor and the true CATE, in that
\begin{align}
\begin{split}
    \text{MSE}(\tau)\ :=&\ \ \|\tau_0-\tau\|^2
    = \text{CAL}(\tau) + \text{DIS}(\tau)\ ,
\end{split}
\label{eq:calibrationDistortion}
\end{align}
with $\text{DIS}(\tau):= E\{var[\tau_0(W)\,|\,\tau(W)]\}$ a quantity we term the distortion of $\tau$. We refer to the above as a \emph{calibration-distortion} decomposition of the MSE. A consequence of the calibration-distortion decomposition is that MSE-consistent CATE estimators are also calibrated asymptotically. However, particularly in settings where the covariates are high-dimensional or the CATE is nonsmooth, the calibration error rate for such predictors can be arbitrarily slow ---  this is discussed further after Theorem \ref{theorem1}.

To interpret $\text{DIS}(\tau)$, we find it helpful to envision a scenario in which a distorted message is passed between two persons. The goal is for Person 2 to discern the value of $\tau_0(w)$, where the value of $w\in \mathcal{W}$ is only known to Person 1. 
Person 1 transmits $w$, which is then distorted through a function $\tau$ and received by Person 2. Person 2 knows the functions $\tau$ and $\tau_0$, and may use this information to try to discern $\tau_0(w)$. If $\tau$ is one-to-one, $\tau_0(w)$ can be discerned by simply applying $\tau_0\circ \tau^{-1}$ to the received message $\tau(w)$. More generally, whenever there exists a function $f$ such that $\tau_0=f\circ \tau$, Person 2 can recover the value of $\tau_0(w)$. For example, if $\tau=\tau_0$ then $f$ is the identity function. If no such function $f$ exists, it may not be possible for Person 2 to recover the value of $\tau_0(w)$. Instead, they may predict $\tau_0(w)$ based on $\tau(w)$ via $\gamma_0(\tau,w)$. Averaged over $W\sim P_W$, the MSE of this approach is precisely $\text{DIS}(\tau)$. 
See Equation~3 in \citet{kuleshov2015calibrated} for a related decomposition of $E\,[\{Y-\tau(X)\}^2]=\text{MSE}(\tau)+E\,[\{Y-\tau_0(X)\}^2]$  derived in the context of probability forecasting.

The calibration-distortion decomposition shows that, at a given level of distortion, better-calibrated treatment effect predictors have lower MSE for the true CATE function. We will explore this fact later in this work when showing that, in addition to improving calibration, our proposed calibration procedure can improve the MSE of CATE predictors.  

\subsection{Calibrating Predictors: desiderata and classical methods}

In most calibration methods, the key goal is to find a function $\theta : \mathbb{R}\rightarrow\mathbb{R}$ of a given predictor $\tau$ such that  $\text{CAL}(\theta \circ \tau)<\text{CAL}(\tau)$, where $\theta \circ \tau$ refers to the composed predictor $w\mapsto \theta(\tau(w))$. A mapping $\theta$ that pursues this objective is referred to as a \emph{calibrator}. Ideally, a calibrator $\theta_n$ for $\tau$ constructed from the dataset $\mathcal{D}_n$ should satisfy the following desiderata:
\begin{enumerate}[ref=\arabic*,label=\underline{\text{Property }\arabic*:},leftmargin=*]
    \item\label{property1} $\text{CAL}(\theta_n\circ \tau)$ tends to zero quickly as $n$ grows;\vspace{-.075in}
    \item\label{property2} $\theta_n\circ \tau$ and $\tau$ are comparably predictive of $\tau_0$.  
\end{enumerate}
Property \ref{property1} states the primary objective of a calibrator, that is, to yield a well-calibrated predictor. Property \ref{property2} requires that the calibrator not destroy the predictive power of the initial predictor in the pursuit of Property \ref{property1}, which would occur if the calibration term in decomposition \eqref{eq:calibrationDistortion} were made small at the cost of dramatic inflation of the distortion term.

In the traditional setting of classification and regression, a natural aim is to learn, for $a \in \{0,1\}$, a predictor $w \mapsto \nu^{(a)}(w)$ of the outcome $Y$ among individuals with treatment $A=a$. The best possible such predictor is given by the treatment-specific outcome regression $w \mapsto \mu_0(a,w)$. For $a \in \{0,1\}$, $\nu^{(a)}$ is said to be calibrated for the outcome regression if $\nu^{(a)}(w) \approx E(Y \mid  \nu^{(a)}(W) = \nu^{(a)}(w), A = a)$ for $P_0$-almost every $w$. Such a calibrated predictor can be obtained using existing calibration methods for regression \citep{huang2020tutorial}, which we review in the next paragraph. It is natural to wonder, then, whether existing calibration approaches can be directly used to calibrate for the CATE. As a concrete example, given predictors $\nu^{(1)}$ and $\nu^{(0)}$ of $\mu_0(1, \cdot)$ and $\mu_0(0, \cdot)$, a natural CATE predictor is the T-learner $\tau := \nu^{(1)} - \nu^{(0)}$. However, even if $\nu^{(1)}$ and $\nu^{(0)}$ are calibrated for their respective outcome regressions, the predictor $\tau$ can still be poorly calibrated for the CATE. Indeed, in settings with treatment-outcome confounding, T-learners can be poorly calibrated when the calibrated predictors $\nu^{(1)}$ and $\nu^{(0)}$ are poor approximations of their respective outcome regressions. As an extreme example, suppose that $\nu^{(a)}$ equals the constant predictor $w \mapsto E(Y \mid A = a) $ for $a \in \{0,1\}$, which is perfectly calibrated for the outcome regression. Then, the corresponding T-learner $ \tau(\cdot) = E(Y \mid A =1) - E(Y \mid A = 0)$ typically has poor calibration for the CATE in observational settings.

 In classification and regression settings \citep{huang2020tutorial}, the most commonly used calibration methods include Platt's scaling \citep{platt1999probabilistic}, histogram binning \citep{zadrozny2001obtaining}, Bayesian binning into quantiles \citep{naeini2015obtaining}, and isotonic calibration \citep{zadrozny2002transforming, niculescu2005obtaining}. Broadly, Platt's scaling is designed for binary outcomes and uses the estimated values of the predictor to fit the logistic regression model \[\text{logit}\,P(Y=1\,|\,\tau(W) = t)= \alpha + \beta t\] with $\alpha,\beta\in\mathbb{R}$. While it typically satisfies Property \ref{property2}, Platt's scaling is based on strong parametric assumptions and, as a consequence, may lead to predictions with significant calibration error, even asymptotically \citep{gupta2020DistrFree}. Nevertheless, Platt's scaling may be preferred when limited data is available. Histogram or quantile binning involves partitioning the sorted values of the predictor into a fixed number of bins. Given an initial prediction, the calibrated prediction is the empirical mean of the observed outcome values within the corresponding prediction bin. A significant limitation of histogram binning is that it requires a priori specification of the number of bins. Selecting too few bins can significantly degrade the predictive power of the calibrated predictor, whereas selecting too many bins can lead to poor calibration. Bayesian binning improves upon histogram binning by considering multiple binning models and their combinations; nevertheless, it still requires pre-specification of binning models and prior distributions.

Isotonic calibration is a histogram binning method that learns the bins from data using isotonic regression, a nonparametric method traditionally used for estimating monotone functions \citep{barlow1972isotonic, martino2019calibration, huang2020tutorial}. Specifically, the bins are selected by minimizing an empirical MSE criterion under the constraint that the calibrated predictor is a nondecreasing monotone transformation of the original predictor. Isotonic calibration is motivated by the heuristic that, for a good predictor $\tau$, the calibration function $\gamma_0(\tau, \cdot)$ should be approximately monotone as a function of $\tau$. For instance, when $\tau = \tau_0$, the map $\tau_0 \mapsto \gamma_0(\tau_0, \cdot) = \tau_0$ is the identity function. 
Despite its popularity and strong performance in practice \citep{zadrozny2002transforming, niculescu2005obtaining, gupta2021distribution}, 
 to date, whether isotonic calibration satisfies distribution-free calibration guarantees remains an open question \citep{gupta2022post}. In this work, we will show that isotonic calibration satisfies a distribution-free calibration guarantee in the sense of Property \ref{property1}. We further establish that Property \ref{property2} holds, in that the isotonic selection criterion ensures that the calibrated predictor is at least as predictive as the original predictor up to negligible error.

\section{Causal Isotonic Calibration} \label{sec:secc4}
In real-world experiments, \citet{dwivedi2020stable} found empirically that state-of-the-art CATE estimators tend to be poorly calibrated. However, strikingly, the authors found that such CATE predictors can often still correctly rank the average treatment effect among subgroups defined by bins of the predicted effects. These findings support the heuristic that the calibration function $\gamma_0(\tau, \cdot)$ is often approximately monotone as a function of the predictor $\tau$. This heuristic makes extending isotonic calibration to the CATE setting especially appealing since the monotonicity constraint ensures that the calibrated predictions preserve the (non-strict) ranking of the original predictions.
 
Inspired by isotonic calibration, we propose a doubly-robust calibration method for treatment effects, which we refer to as \emph{causal isotonic calibration}. Causal isotonic calibration takes a given predictor trained on some dataset and performs calibration using an independent (or hold-out) dataset.  Mechanistically, causal isotonic calibration first automatically learns uncalibrated regions of the given predictor. Calibrated predictions are then obtained by consolidating individual predictions within each region into a single value using a doubly-robust estimator of the ATE. In addition, we introduce a novel data-efficient variant of calibration which we refer to as cross-calibration. In contrast with the standard calibration approach, \emph{causal isotonic cross-calibration} takes cross-fitted predictors and outputs a single calibrated predictor obtained using all available data. Our methods can be implemented using standard isotonic regression software.  

Let $\tau$ be a given treatment effect predictor assumed, for now, to have been built using an external dataset, and suppose that $\mathcal{D}_n$ is the available calibration dataset. In general, we can calibrate the predictor $\tau$ using regression-based calibration methods by employing an appropriate surrogate outcome for the CATE. For both experimental and observational settings, a surrogate outcome with favorable efficiency and robustness properties is the pseudo-outcome $\chi_0(O)$ defined via the mapping 
\begin{equation}
    \chi_0:o\mapsto\tau_0(w) + \frac{a - \pi_0(w)}{\pi_0(w)[1-\pi_0(w)]} \left[ y - \mu_0(a,w) \right],\label{eqn::pseudooutcome}
\end{equation} with $o:=(w,a,y)$ representing a realization of the data unit. This pseudo-outcome has been used as surrogate for the CATE in previous methods for estimating $\tau_0$, including the DR-learner \citep{luedtke2016statistical,kennedy2020optimal}. If $\chi_0$ were known, an external predictor $\tau$ could be calibrated using $\mathcal{D}_n$ by isotonic regression of the pseudo-outcomes $\chi_0(O_1),\chi_0(O_2),\ldots,\chi_0(O_n)$ onto the calibration sample predictions $\tau(W_1),\tau(W_2),\ldots,\tau(W_n)$. However, $\chi_0$ depends on $\pi_0$ and $\mu_0$, which are usually unknown and must  be estimated.

A natural approach for calibrating treatment effect predictors using isotonic regression is as follows. 
First, define $\chi_n$ as the estimated pseudo-outcome function based on estimates $\mu_n$ and $\pi_n$ derived from $\mathcal{D}_{n}$. Then, a calibrated predictor is given by $\theta_n \circ \tau$, where the calibrator $\theta_n$ is found via isotonic regression as a minimizer over $\mathcal{F}_{iso} := \{\theta:\mathbb{R} \rightarrow \mathbb{R}; \; \theta \text{ is monotone  nondecreasing}\}$ of the empirical least-squares risk function  
\begin{equation*}
\label{eq7:cal}
\theta \mapsto \frac{1}{n}\sum_{i=1}^n\left[\chi_n(O_i) - \theta\circ\tau(W_i)\right]^2\ .
\end{equation*}
However, this optimization problem requires a double use of $\mathcal{D}_{n}$: once, for creating the pseudo-outcomes $\chi_n(O_i)$, and a second time, in the calibration step. This double usage could lead to over-fitting \citep{kennedy2020optimal}, and so we recommend obtaining pseudo-outcomes via sample splitting or cross-fitting. Sample splitting involves randomly partitioning  $\mathcal{D}_{n}$ into $\mathcal{E}_m \cup \mathcal{C}_{\ell}$, with $\mathcal{E}_m$ used to estimate $\mu_0$ and $\pi_0$, and $\mathcal{C}_{\ell}$ used to carry out the calibration step --- see  Algorithm \ref{alg:cic} for details. Cross-fitting improves upon sample splitting by using all available data to estimate $\mu_0$ and $\pi_0$ as well as to carry out the calibration step. Algorithm \ref{alg:cicexternal::crossfit}, outlined in Appendix \ref{appendix::alg}, is the cross-fitted variant of Algorithm \ref{alg:cic}.
  \setlength{\dblfloatsep}{0pt} 
 
 \begin{algorithm}[!htb]
\begin{algorithmic}[1]
\caption{Causal isotonic calibration}
 \label{alg:cic}\vspace{.1in}
\REQUIRE predictor $\tau$, training data $\mathcal{E}_m$, calibration data $\mathcal{C}_{\ell} $ 
\vspace{.05in}
\STATE obtain estimate $\chi_{m}$ of $\chi_0$ using $\mathcal{E}_m$;
\STATE perform isotonic regression to find
\vspace{-.05in}
\[\theta_{n}^* = \argmin_{\theta \in \mathcal{F}_{iso}}  \sum_{i \in \mathcal{I}_{\ell}}[\chi_m(O_{i})-\theta \circ \tau (W_{i})]^2\]
\vspace{-.15in}

with $\mathcal{I}_{\ell}$ the set of indices for observations in $\mathcal{C}_{\ell} \subset \mathcal{D}_n$;
\STATE set $\tau_{n}^* := \theta_{n}^* \circ \tau$.\label{alg::cic::stepminimizer}
\vspace{.05in}
\ENSURE $\tau_{n}^*$
\end{algorithmic}
\vspace{.05in}
\end{algorithm}

In practice, the external dataset used to construct $\tau$ for input into Algorithm~\ref{alg:cic} is likely to arise from a sample splitting approach wherein a large dataset is split in two, with one half used to estimate $\tau$ and the other to calibrate it. This naturally leads to the question of whether there is an approach that fully utilizes the entire dataset for both fitting an initial estimate of $\tau_0$ and calibration. Algorithm \ref{alg:ciccrossfit} describes causal isotonic cross-calibration, which provides a means to accomplish precisely this. In brief, this approach applies Algorithm~\ref{alg:cic} a total of $k$ times on different splits of the data, where for each split an initial predictor of $\tau_0$ is fitted based on the first subset of the data and this predictor is calibrated using the second subset. These $k$ calibrated predictors are then aggregated via a  pointwise median. 
Interestingly, other aggregation strategies, such as pointwise averaging, can lead to uncalibrated predictions \citep{gneiting2013combining,rahaman2020uncertainty}. A computationally simpler variant of Algorithm \ref{alg:ciccrossfit} is given by Algorithm \ref{alg:ciccrossfit::pooled}. In this implementation, a single isotonic regression is performed using the pooled out-of-fold predictions; this variant may also yield more stable performance in finite-samples than Algorithm~\ref{alg:ciccrossfit} --- see Section 2.1.2 of \citet{xu2022calibration} for a related discussion in the context of debiased machine learning.

\begin{algorithm}[!htb]
\begin{algorithmic}[1]
\caption{Causal isotonic cross-calibration (unpooled)}
\label{alg:ciccrossfit}\vspace{.1in}
\REQUIRE dataset $\mathcal{D}_n$, \# of cross-fitting splits $k$
\vspace{.05in}
\STATE partition $\mathcal{D}_n$ into datasets $\mathcal{C}^{(1)},\mathcal{C}^{(2)},\ldots,\mathcal{C}^{(k)}$;
\FOR {$s=1,2,\ldots,k$}
\STATE set $\mathcal{E}^{(s)} := \mathcal{D}_n \backslash \mathcal{C}^{(s)}$;
\STATE get initial predictor $\tau_{n,s}$ of $\tau_0$ using $\mathcal{E}^{(s)}$;
\STATE get calibrated predictor $\tau_{n,s}^*$ via Alg.~\ref{alg:cic} using predictor $\tau_{n,s}$, training data $\mathcal{E}^{(s)}$, and calibration data $\mathcal{C}^{(s)}$;
\ENDFOR
\STATE set ${\tau}_n^*:w\mapsto\normalfont\text{median}\{\tau^*_{n,1}(w),\tau_{n,2}^*(w),\ldots,\tau_{n,k}^*(w)\}$.
 
\ENSURE $\tau_n^*$
\end{algorithmic}
\end{algorithm}

 \begin{algorithm}[!htb]
\begin{algorithmic}[1]
\caption{Causal isotonic cross-calibration (pooled)}
\label{alg:ciccrossfit::pooled}\vspace{.1in}
\REQUIRE   dataset $\mathcal{D}_n$, \# of cross-fitting splits $k$
\vspace{.05in}
\STATE partition $\mathcal{D}_n$ into datasets $\mathcal{C}^{(1)},\mathcal{C}^{(2)},\ldots,\mathcal{C}^{(k)}$;
\FOR {$s = 1,2,\ldots,k$}
\STATE let $j(i)=s$ for each $i\in \mathcal{C}^{(s)}$;
\STATE set $\mathcal{E}^{(s)} := \mathcal{D}_n \backslash \mathcal{C}^{(s)}$;
\STATE get estimate $\chi_{n,s}$ of $\chi_0$ from $\mathcal{E}^{(s)}$;
\STATE get initial predictor $\tau_{n,s}$ of $\tau_0$ from $\mathcal{E}^{(s)}$;
\ENDFOR
\STATE perform isotonic regression using pooled out-of-fold predictions to find
\vspace{-.1in}
\[\theta_n^*=\argmin_{\theta \in \mathcal{F}_{iso}}\sum_{i=1}^n\left[\chi_{n,j(i)}(O_i)-(\theta \circ \tau_{n,j(i)} )(W_i)\right]^2;\]
\vspace{-.15in}
\STATE set $\tau_{n,s}^* := \theta_n^* \circ \tau_{n,s}$ for $s=1,2,\dots, k$;
\STATE set $\tau_n^*:w\mapsto \text{median}\{\tau_{n,1}^*(w), \tau_{n,2}^*(w), \dots, \tau_{n,k}^*(w)\}$.
 
\ENSURE $\tau_n^*$
\end{algorithmic}
\vspace{.05in}
\end{algorithm}

\section{Large-Sample Theoretical Properties}\label{sec:secc5}

We now present theory for causal isotonic calibration. We obtain results for causal isotonic calibration described by Algorithm \ref{alg:cic} applied to a fixed predictor $\tau$.  We also establish MSE guarantees for the calibrated predictor and argue that the proposed calibrator satisfies Properties \ref{property1} and \ref{property2}. We extend our results to the procedure of Algorithm \ref{alg:ciccrossfit}. 

For ease of presentation, we only establish theoretical results for the case where the nuisance estimators are obtained using sample splitting. With minor modifications, our results can be readily extended to cross-fitting by arguing along the lines of \citet{newey2018cross}.
In that spirit, we assume that the available data $\mathcal{D}_n$ is the union of a training dataset $\mathcal{E}_m$ and a calibration dataset $\mathcal{C}_{\ell}$ of sizes $m$ and $\ell$, respectively, with $n = m + \ell$ and $\min\{m, \ell\} \rightarrow \infty$ as $n\rightarrow\infty$. Let $\tau_n^*$ be the calibrated predictor obtained from Algorithm \ref{alg:cic} using $\tau$, $\mathcal{E}_m$ and $\mathcal{C}_{\ell}$ where the estimated pseudo-outcome $\chi_m$ is obtained by substituting estimates $\pi_m$ and $\mu_m$ of $\pi_0$ and $\mu_0$ into \eqref{eqn::pseudooutcome}.

\begin{condition}[bounded outcome support]
\label{assumption::A1}
The $P$-support $\mathcal{Y}$ of $Y$ is a uniformly bounded subset of $\mathbb{R}$.
\end{condition}

\begin{condition}[positivity]
\label{assumption::A2} There exists $\epsilon>0$ such that $P(\epsilon < \pi_{0}(W) < 1-\epsilon)=1$.
\end{condition}

\begin{condition}[independence]
\label{assumption::A3}
Estimators $\pi_m$ and $\mu_m$ do not use any data in $\mathcal{C}_{\ell}$.
\end{condition}

\begin{condition}[bounded range of $\pi_m$, $\mu_m$, $\tau$]
\label{assumption::A4}
There exist $0<\eta,\alpha<\infty$ such that 
$P(\eta<\pi_m(W)<1-\eta)=P(|\mu_m(A,W)|<\alpha)=P(|\tau(W)|<\alpha)=1$ for $m=1,2,\ldots$
\end{condition}
 
\begin{condition} [bounded variation of best predictor]
The function $\theta_0: \mathbb{R} \mapsto \mathbb{R}$ such that $\theta_0 \circ \tau = \gamma_0(\tau, \cdot) $ is of bounded total variation. \label{assumption::A6}
\end{condition}
 
It is worth noting that the initial predictor and its best monotone transformation can be arbitrarily poor CATE predictors. Condition \ref{assumption::A1} holds trivially when outcomes are binary, but even continuous outcomes are often known to satisfy fixed bounds (e.g., physiologic bound, limit of detection of instrument) in applications. Condition \ref{assumption::A2} is standard in causal inference and requires that all individuals have a positive probability of being assigned to either treatment or control. Condition \ref{assumption::A3} follows as a direct consequence of the sample splitting approach, because the estimators are obtained from an independent sample from the data used to carry the calibration step. Condition \ref{assumption::A4} requires that the estimators of the outcome regression and propensity score be bounded; this can be enforced, for example, by threshholding when estimating these regression functions. Condition \ref{assumption::A6}  excludes cases in which the best possible predictor of the CATE given only the initial predictor $\tau$ has pathological behavior, in the sense that it has infinite variation norm as a (univariate) mapping of $\tau$. We stress here that isotonic regression is used only as a tool for calibration, and our theoretical guarantees do not require any monotonicity on components of the data-generating mechanism --- for example, $\gamma_0(\tau,w)$ need not be monotone as a function of $\tau(w)$. 
 
The following theorem establishes the calibration rate of the predictor $\tau_n^*$ obtained using causal isotonic calibration.
\begin{theorem}[$\tau_n^*$ is well-calibrated]
\label{theorem1}
Under Conditions \ref{assumption::A1}--\ref{assumption::A6}, as $n\rightarrow\infty$, it holds that
\[\normalfont\text{CAL}(\tau_n^*)= O_P\left(\ell^{-2/3}  +  \norm{(\pi_m - \pi_0)(\mu_m - \mu_0)}^2   \right).\]
\end{theorem}
The calibration rate can be expressed as the sum of an oracle calibration rate and the rate of a second-order cross-product bias term involving nuisance estimators. Notably, the causal isotonic calibrator rate can satisfy Property \ref{property1} at the oracle rate $\ell^{-2/3}$ so long as $\norm{(\pi_{m}-\pi_0)(\mu_{m} - \mu_0)}$ shrinks no slower than $\ell^{-1/3}$, which requires that one or both of $\pi_0$ and $\mu_0$ is estimated well in an appropriate sense. 
 If $\pi_0$ is known, as in most randomized experiments, the fast calibration rate of $\ell^{-2/3}$ can be achieved even when $\mu_m$ is inconsistent, thereby providing distribution-free calibration guarantees irrespective of the smoothness of the outcome regression or dimension of the covariate vector. 
When $\pi_0$ is unknown, the oracle rate of $\ell^{-2/3}$ may not be achievable if the propensity score and outcome regression are insufficiently smooth relative to the dimension of the covariate vector \citep{kennedy2020optimal, KennedyMiximax}.

It is interesting to contrast the calibration guarantee in Theorem \ref{theorem1} with existing MSE guarantees for DR-learners \citep{kennedy2020optimal} since, in view of \eqref{eq:calibrationDistortion}, they also provide calibration guarantees. 
While the MSE estimation rates for the CATE depend on the dimension and smoothness of $\tau_0$, the curse of dimensionality for our calibration rates only manifests itself in the doubly-robust cross-remainder term that involves nuisance estimation rates. For instance, when $\ell = m = n/2$, if $\pi_0$ and $\mu_0$ are known to be H\"{o}lder smooth with exponent $\alpha \geq 1$, the calibration rate implied by Theorem \ref{theorem1} with minimax optimal nuisance estimators is, up to logarithmic factors, $\ell^{-2/3} + \ell^{-4\alpha/(2\alpha+d)}$. In contrast, if $\tau_0$ is known to be  H\"{o}lder smooth with exponent $\beta \geq 1$, a minimax optimal estimator of $\tau_0$ is only guaranteed to achieve an MSE, and therefore calibration, rate of $\ell^{-2\beta/(2\beta + d)} + \ell^{-4\alpha/(2\alpha+d)}$ \citep{KennedyMiximax}. When the nuisance smoothness satisfies $\alpha \geq d/4$, causal isotonic calibration can achieve the oracle calibration rate of $\ell^{-2/3}$, whereas a minimax optimal CATE estimator is only guaranteed to achieve the same calibration rate under the stringent condition that the smoothness of $\tau_0$ satisfies $\beta \geq d$.

The following theorem states that the predictor obtained by taking pointwise medians of calibrated predictors is also calibrated.

\begin{theorem}[Pointwise median preserves calibration]
  \label{theoremmedian}  Let $\tau_{n,1}^{*},\tau_{n,2}^{*},\ldots,\tau_{n,k}^{*}$ be predictors, and define pointwise $\tau_n^*(w):=\normalfont\text{median}\{\tau_{n,1}^{*}(w),\tau_{n,2}^{*}(w),\ldots,\tau_{n,k}^{*}(w)\}$.  Then,\vspace{-.12in}
 \[\normalfont\text{CAL}( \tau_n^*) \leq k\sum_{s=1}^k \text{CAL}(\tau_{n,s}^{*})\ ,\]
 where the median operation is defined as in Section \ref{sec::notation}. 
\end{theorem}
 Under similar conditions, Theorem \ref{theoremmedian} combined with a generalization of Theorem \ref{theorem1} that handles random $\tau$ (see Theorem \ref{theorem1::RandomPredictor} in Appendix \ref{appendix::theorem1Generaliz}) establishes that a predictor $\tau_n^*$ obtained using causal isotonic cross-calibration (Algorithm \ref{alg:ciccrossfit}) has calibration error $\normalfont\text{CAL}( \tau_n^*)$ of order \[  O_P\left(n^{-2/3}+ \max_{1\leq s\leq k} \norm{(\pi_{n,s} - \pi_0)(\mu_{n,s} - \mu_0)}^2   \right)\]as $n\rightarrow\infty$, where $\mu_{n,s}$ and $\pi_{n,s}$ are the outcome regression and propensity score estimators obtained after excluding the $s^{th}$ fold of the full dataset.  In fact, Theorem  \ref{theoremmedian} is valid for any calibrator of the form $\tau_n^*:w\mapsto\tau_{n,s_n(w)}^{*}(w)$, where $s_n(w)$ is any random selector that may depend on the covariate value $w$. This suggests that the calibration rate for the median-aggregated calibrator implied by Theorem \ref{theoremmedian} is conservative as it also holds for the worst-case oracle selector that maximizes calibration error.

We now establish that causal isotonic calibration satisfies Property \ref{property2}, that is, it maintains the predictive accuracy of the initial predictor $\tau$. In what follows, predictive accuracy is quantified in terms of MSE. At first glance, the calibration-distortion decomposition appears to raise concerns that causal isotonic calibration may distort $\tau$ so much that the predictive accuracy of $\tau_n^*$ may be worse than that of $\tau$. This possibility may seem especially concerning given that the ouput of isotonic regression is a step function, so that there could be many $w,w'\in\mathcal{W}$ such that $\tau(w)\not=\tau(w')$ but $\tau_n^*(w)=\tau_n^*(w')$. The following theorem alleviates this concern by establishing that, up to a remainder term that decays with sample size, the MSE of $\tau_n^*$ is no larger than the MSE of the initial CATE predictor $\tau$. A consequence of this theorem is that causal isotonic calibration does not distort $\tau$ so much as to destroy its predictive performance. To derive this result, we leverage that $\tau_n^*$ is in fact a misspecified DR-learner of the univariate CATE function $\gamma_0(\tau, \cdot)$. While isotonic calibrated predictors are calibrated even when $\gamma_0(\tau, \cdot)$ is not a monotone function of $\tau$, we stress that misspecified DR-learners for $\gamma_0(\tau, \cdot)$ are typically uncalibrated. 

 In the theorem below, we define the best isotonic approximation of the CATE given the initial predictor $\tau$ as \[\tau_0^* := \argmin_{\theta \circ \tau: \theta \in \mathcal{F}_{iso}} \norm{\tau_0 - \theta \circ \tau}.\]

\begin{theorem}[Causal isotonic calibration does not inflate MSE much]
 \label{theorem2}
Under Conditions \ref{assumption::A1}--\ref{assumption::A6}, 
\[ \norm{\tau_{n}^* -  \tau_0^*}=O_P \left( \ell^{-1/3}+\norm{(\pi_m - \pi_0)(\mu_m-\mu_0)}\right)\] 
as $n\rightarrow\infty$. As such, as $n\rightarrow\infty$, the inflation in root MSE from causal isotonic calibration satisfies
\begin{align*}
     &\sqrt{\normalfont\text{MSE}(\tau_{n}^*)}-\sqrt{\normalfont\text{MSE}(\tau)}  \le O_P \left( \ell^{-1/3}+ \norm{(\pi_m - \pi_0)(\mu_m-\mu_0)} \right).
\end{align*}
\end{theorem}
A similar MSE bound can be established for causal isotonic cross-calibration as defined in Algorithm \ref{alg:ciccrossfit}. 
 
\section{Simulation Studies} \label{sec:secc6}

\subsection{Data-Generating Mechanisms}

We examined the behavior of our proposal under two data-generating mechanisms. The first mechanism (Scenario 1) includes a binary outcome whose conditional mean is an additive function (on the logit scale) of non-linear transformations of four confounders with treatment interactions. The second mechanism (Scenario 2) includes instead a continuous outcome with conditional mean linear on covariates and treatment interactions, with more than 100 covariates of which only 20 are true confounders. In both scenarios, the propensity score follows a logistic regression model. All covariates were independent and uniformly distributed on $(-1,+1)$. Sample sizes $1,000$, $2,000$, $5,000$ and $10,000$ were considered. Further details are given in Appendix \ref{subsecApp:simdatainfo}.

\subsection{CATE Estimation}

We employed the DR-learner algorithm, as outlined by \citet{kennedy2020optimal}, in combination with different supervised learning algorithms to generate uncalibrated predictors of the CATE. In Scenario 1, to estimate the CATE, we implemented gradient-boosted regression trees (GBRT) with maximum depths equal to 2, 5, and 8 \citep{chen2016xgboost}, random forests (RF) \citep{breiman2001random}, generalized linear models with lasso regularization (GLMnet)  \citep{friedman2010regularization}, generalized additive models (GAM) \citep{wood2017introducing}, and multivariate adaptive regression splines (MARS) \citep{friedman1991multivariate}. In Scenario 2, we implemented RF, GLMnet, and a combination of variable screening with lasso regularization followed by GBRT with maximum depth determined via cross-validation. We used the implementation of these estimators found in R package \texttt{sl3} \citep{coyle2021sl3-rpkg}. Causal isotonic cross-calibration was implemented using the variant outlined in Algorithm \ref{alg:ciccrossfit::pooled}. Further details are given in Appendix \ref{subsecApp:impcic}.

\subsection{Performance Metrics} \label{subseq:perfm}

We evaluated the performance of each causal isotonic cross-calibrated predictor relative to its corresponding uncalibrated predictor using three metrics: the calibration measure defined in \eqref{eq2:cal}, MSE, and the calibration bias within bins defined by the first and last prediction deciles. The calibration bias within bins is given by the measure in \eqref{eq1:cal} standardized by the probability of falling within each bin. For each simulation iteration, the metric was estimated empirically using an independent sample $\mathcal{V}$ of size $n_{\mathcal{V}}=10^4$. These metric estimates were then averaged across 1000 simulations. Details on these metrics are provided in Appendix \ref{subsecApp:perfmet}.

\subsection{Simulation Results}

Results from Scenario 1 are summarized in Figure \ref{simfig:scenario1}. The predictors based on GLMnet and GAM happened to be well-calibrated, and so, causal isotonic calibration did not lead to substantial improvements in calibration error. In contrast, causal isotonic calibration of RF, MARS, and GBRT substantially decreased its calibration error, regardless of tree depth and sample size. In terms of MSE, calibration improved the predictive performance of RF, MARS, GBRT, and preserved the performance of GLMnet and GAM. The calibration bias within bins of prediction was generally smaller after calibration, with a more notable improvement on MARS, RF, and GBRT --- see Table \ref{tabApp:bwb1} in Appendix \ref{appendix::simResults}.

 Results from Scenario 2 are summarized in Figure \ref{simfig:scenario2}. The predictors based on RF and GBRT with GLMnet screening were poorly calibrated, and causal isotonic calibration substantially reduced their calibration error. Calibration did not noticeably change the already small calibration error of the GLMnet predictions; however, calibration substantially reduced the calibration error within quantile bins of its predictions --- see Table \ref{tabApp:bwb2} in Appendix \ref{appendix::simResults}. Finally, with respect to MSE, causal isotonic calibration improved the performance of RF and GBRT with variable screening, and yielded similar performance to GLMnet. 
 
 In Figure \ref{figApp:externalcal} of Appendix \ref{appendix::simResults}, we compared calibration performance using hold-out sets to cross-calibration and found substantial improvements in MSE and calibration by using cross-calibration.

 \begin{figure}[H]
\vskip 0.2in
\begin{center}
\begin{subfigure}[H]{0.45\linewidth}
\centerline{\includegraphics[width=\columnwidth]{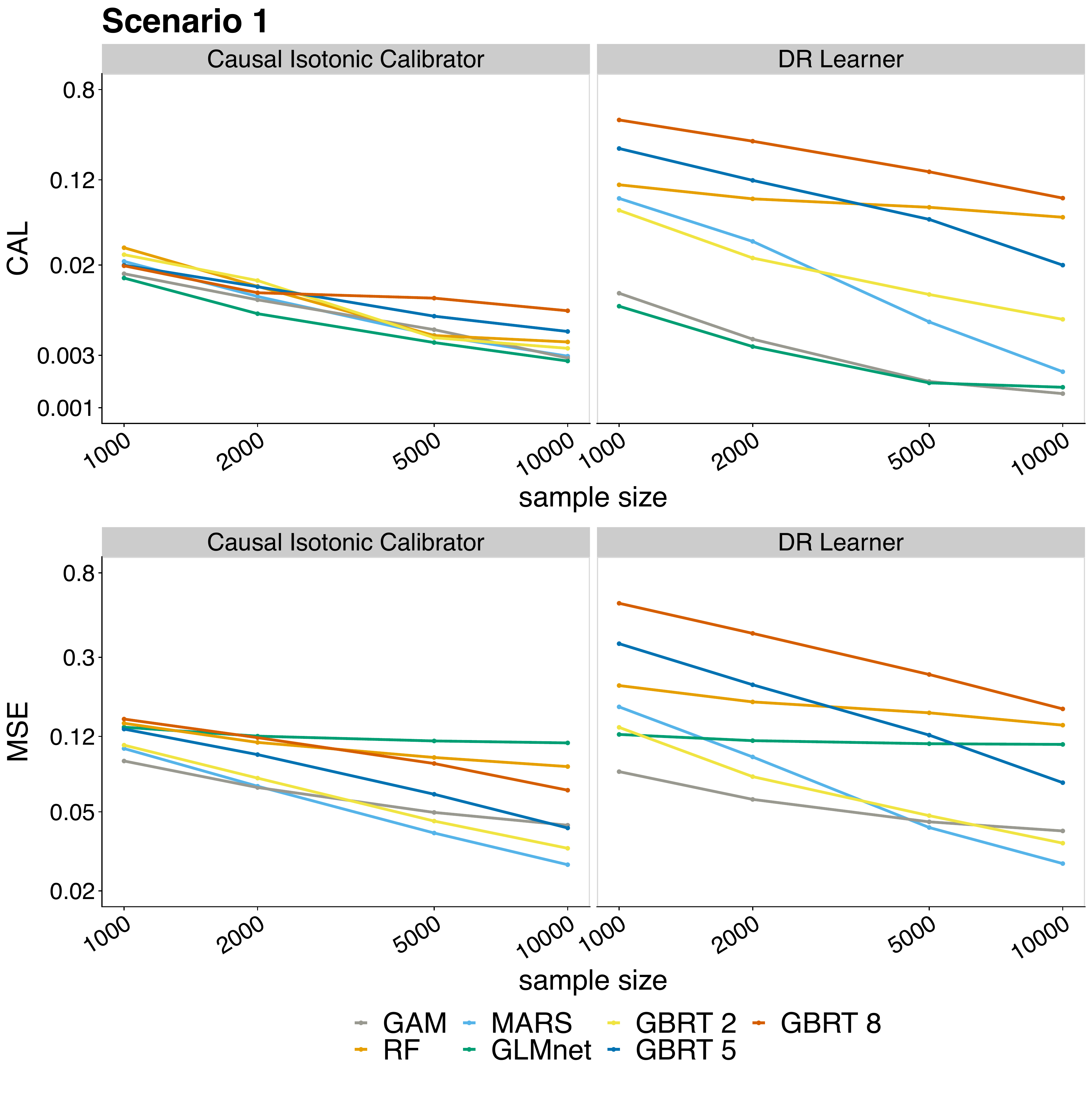}}
\caption{Calibration error and MSE in Scenario 1. The panels show the calibration error (top) and MSE (bottom) using the calibrated (left) and uncalibrated (right) predictors as a function of sample size. Both calibration error and the MSE were standardized by $\text{Var}\left(Y(1) - Y(0)\right)$. The y-axes and x-axis are on a log scale.}
\label{simfig:scenario1}
\end{subfigure}\hspace{0.5cm}
\begin{subfigure}[H]{0.45\linewidth}
\centerline{\includegraphics[width=\columnwidth]{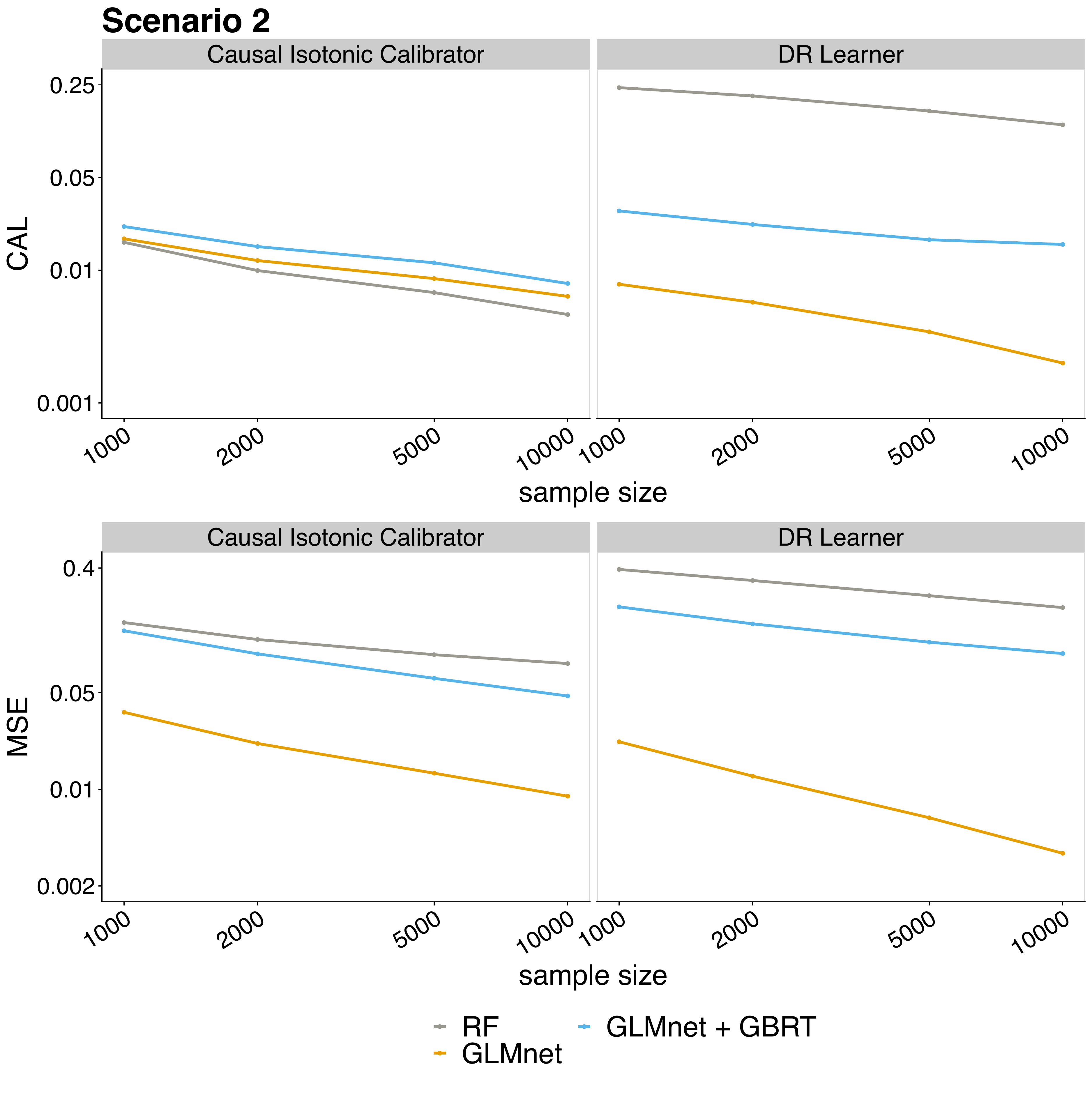}}
\caption{Calibration error and MSE in Scenario 2. The panels show the calibration error (top) and MSE (bottom) using the calibrated (left) and uncalibrated (right) predictors as a function of sample size. Both calibration error and the MSE were standardized by $\text{Var}\left(Y(1) - Y(0)\right)$. The y-axes and x-axis are on a log scale.}
\label{simfig:scenario2}
\end{subfigure} 

\end{center}
\vskip -0.2in
\end{figure}

\section{Conclusion} \label{sec:secc7}

In this work, we proposed causal isotonic calibration as a novel method to calibrate treatment effect predictors. In addition, we established that the pointwise median of calibrated predictors is also calibrated. This allowed us to develop a data-efficient variant of causal isotonic calibration using cross-fitted predictors, thereby avoiding the need for a hold-out calibration dataset. Our proposed methods guarantee that, under minimal assumptions, the calibration error defined in \eqref{eq1:cal} vanishes at a fast rate of $\ell^{-2/3}$ with little or no loss in predictive power, where $\ell$ denotes the number of observations used for calibration. This property holds regardless of how well the initial predictor $\tau$ approximates the true CATE function. To our knowledge, our method is the first in the literature to directly calibrate CATE predictors without requiring trial data or parametric assumptions. Potential applications of our method include data-driven decision-making with strong robustness guarantees. In future work, it would be interesting to study whether pairing causal isotonic cross-calibration with conformal inference \citep{lei2021conformal} leads to improved ITE prediction intervals, and whether causal isotonic calibration and shape-constrained inference methods \citep{westling2020unified} can be used to construct confidence intervals for $\gamma_0(\tau_n^*, \cdot)$.

Our method has limitations. Its calibration guarantees require that either $\mu_0$ or $\pi_0$ be estimated sufficiently well. Flexible learning methods can be used to satisfy this condition. If $\pi_0$ is known, this condition can be trivially met. Hence, our method can be readily used to calibrate CATE predictors and characterize HTEs in clinical trials. For proper calibration, our method requires all confounders to be measured and adjusted for. In future work, it will be important to study CATE calibration in the context of unmeasured confounding. Our strategy could be adapted to construct calibrators for general learning tasks, including E-learning of the conditional relative risk \citep{JiangEntropy, LuedtkeEntropyDiscussion}, proximal causal learning \citep{tchetgen2020introduction, sverdrup2023proximal}, and instrumental variable-based learning \citep{okui2012doubly, syrgkanis2019machine}. 

In simulations, we found that causal isotonic cross-calibration led to well-calibrated predictors without sacrificing predictive performance; benefits were especially prominent in high-dimensional settings and for tree-based methods. This is of particularly high relevance given that regression trees have become popular for CATE estimation, due to both their flexibility \citep{athey2016recursive} and interpretability \citep{lee2020causal}. We also found that cross-calibration substantially improved the MSE of the calibrated predictor relative to hold-out set approaches. In some cases, cross-calibration even improved upon the MSE of the uncalibrated predictor.  

Though our focus was on treatment effect estimation, our theoretical arguments can be readily adapted to provide guarantees for isotonic calibration in regression and classification problems. 
Hence, we have provided an affirmative answer to the open question of whether it is possible to establish distribution-free calibration guarantees for isotonic calibration \citep{gupta2022post}.

\noindent\textbf{Acknowledgements.} Research reported in this publication was supported by NIH grants DP2-LM013340 and  R01-HL137808. The content is solely the responsibility of the authors and does not necessarily represent the official views of the funding agencies.

\newpage
\bibliography{bibliography.bib}
 
 \onecolumn
 \appendix

\section{Implementation of algorithms in R}

R code implementing causal isotonic calibration with user-supplied (cross-fitted) nuisance estimates and predictions is provided in the Github package \textit{causalCalibration} and can be found at \hyperlink{https://github.com/Larsvanderlaan/causalCalibration}{https://github.com/Larsvanderlaan/causalCalibration}.

\section{Algorithm for causal isotonic calibration with cross-fitted nuisance estimates}
\label{appendix::alg}
\begin{algorithm}[H]
\begin{algorithmic}[1]
\caption{Causal isotonic calibration (cross-fitted nuisances)}
\label{alg:cicexternal::crossfit}\vspace{.1in}
\REQUIRE  predictor $\tau$, dataset $\mathcal{D}_n$, \# of cross-fitting splits $k$
\vspace{.05in}
\STATE partition $\mathcal{D}_n$ into datasets $\mathcal{T}^{(1)},\mathcal{T}^{(2)},\ldots,\mathcal{T}^{(k)}$;
\FOR {$s = 1,2,\ldots,k$}
\STATE let $j(i)=s$ for each $i\in \mathcal{T}^{(s)}$;
\STATE get estimate $\chi_{n,s}$ of $\chi_0$ from $\mathcal{D}_n \backslash \mathcal{T}^{(s)}$;
\ENDFOR
\STATE perform isotonic regression using pooled out-of-fold estimates to find
\vspace{-.1in}
\[\theta_n^*=\argmin_{\theta \in \mathcal{F}_{iso}} \frac{1}{n}\sum_{i=1}^n\left[\chi_{n,j(i)}(O_i)-(\theta \circ \tau )(W_i)\right]^2;\]
\vspace{-.15in}
\STATE set $\tau_{n}^* := \theta_n^* \circ \tau$;
\ENSURE $\tau_n^*$
\end{algorithmic}
\vspace{.05in}
\end{algorithm}

\section{Technical proofs}
 
Unless stated otherwise, the function $\tau_n^*$ denotes a calibrated predictor obtained using Algorithm \ref{alg:cic} with a predictor $\tau$, training dataset $\mathcal{E}_m$, and calibration dataset $\mathcal{C}_{\ell} = \mathcal{D}_n \backslash \mathcal{E}_m$ as described in Section \ref{sec:secc5}. 

\subsection{Notation \& definitions}

Let $\mathcal{T} := \{\tau(w): w \in \mathcal{W}\}$ denote the range of the predictor $\tau$, which is a bounded subset of $\mathbb{R}$ by Condition \ref{assumption::A4}. We redefine  $\mathcal{F}_{iso} \subset \{\theta: \mathcal{T} \rightarrow \mathbb{R}; \theta \text{ is monotone nondecreasing}\}$ to denote the family of nondecreasing functions on $\mathcal{T}$ uniformly bounded by \[B  :=  \sup_{ m \in \mathbb{N}}\sup_{\mathcal{E}_m}\sup_{o \in \mathcal{O}}\left[|\chi_0(o)| + |\chi_m(o)|\right],  \] where the second supremum is over all possible realizations of the training dataset $\mathcal{E}_m$. We necessarily have that $B$ is nonrandom and finite by Lemma \ref{lem:lem2}. Redefining $\mathcal{F}_{iso}$ to be bounded allows us to directly apply certain maximal inequalities for empirical processes indexed by $\mathcal{F}_{iso}$. Since the isotonic regression estimator is obtained by locally averaging the pseudo-outcome $\chi_m$ \citep{barlow1972isotonic}, the unconstrained isotonic regression solution satisfies this bound and falls in the interior of this class almost surely. Moreover, $\mathcal{F}_{iso}$ is a convex subset of the space of monotone  nondecreasing functions. Let  $\mathcal{F}_{TV} \subset \{\theta: \mathbb{R} \rightarrow \mathbb{R}; \theta \text{ is of bounded variation}\}$ denote the space of functions with total variation uniformly bounded by three times the total variation of the function $\theta_0$ where $\theta_0$ is as in condition \ref{assumption::A6}. Additionally, let
$\mathcal{F}_{\tau,iso} := \left\{ \theta \circ \tau :\mathcal{W} \rightarrow \mathbb{R}; \theta \in \mathcal{F}_{iso} \right\}$ be the family of functions obtained by composing nondecreasing functions in $\mathcal{F}_{iso}$ with $\tau$, and let $\mathcal{F}_{\tau,TV}:= \{\theta \circ \tau: \mathcal{W}\rightarrow \mathbb{R} ; \theta \in \mathcal{F}_{TV}\}$ be the family of functions obtained by composing functions in $\mathcal{F}_{TV}$ with $\tau$. Let $\mathcal{F}_{Lip,m} := \left\{ o\mapsto [ \tau_2(w) - \tau_1(w) ][\chi_m(o)- \tau_2(w) ]: \mathcal{O} \rightarrow \mathbb{R}; \tau_2 \in \mathcal{F}_{\tau,TV}, \tau_1 \in \mathcal{F}_{\tau,iso} \right\}$, where $\chi_m$ is the estimated pseudo-outcome function. Finally, for a function class $\mathcal{F}$, let $N(\epsilon,\mathcal{F},L_2(P))$ denote the $\epsilon-$covering number \citep{van1996weak} of $\mathcal{F}$ and define the uniform entropy integral of $\mathcal{F}$ by  
\begin{equation*}
\mathcal{J}(\delta,\mathcal{F}):= \int_{0}^{\delta} \sup_{Q}\sqrt{\log N(\epsilon,\mathcal{F},L_2(Q))}\,d\epsilon\ ,
\end{equation*}
where the supremum is taken over all discrete probability distributions $Q$. In contrast to the definition provided in \citet{van1996weak}, we do not define the uniform entropy integral relative to an envelope function for the function class $\mathcal{F}$. We can do this since all function classes we consider are uniformly bounded. Thus, any uniformly bounded envelope function will only change the uniform entropy integral as defined in \citet{van1996weak} by a constant. 

In the results below, we will use the following empirical process notation: for a $P-$measurable function $f$, we denote $\int f(o)dP(o)$ by $Pf$, and so, letting $P_{\ell}$ denote the empirical distribution of $\mathcal{C}_{\ell}$, $P_{\ell}f$ equals $\frac{1}{\ell}\sum_{i \in \mathcal{I}_{\ell}} f(O_i)$ with $\mathcal{I}_{\ell}$ indexing observations of $\mathcal{C}_{\ell} \subset \mathcal{D}_n$. We also let $\norm{f}_P^2 := Pf^2$; to simplify notation, we omit the dependency in $P$ and use $\norm{f}^2$ instead of $\norm{f}_P^2$. Finally, for two quantities $x$ and $y$, we use the expression  $x \lessapprox y$ to mean that $x$ is upper bounded by $y$ times a universal constant that may only depend on global constants that appear in conditions \ref{assumption::A1}-\ref{assumption::A6}

\subsection{Technical lemmas}

The following lemma is a key component of our proof of Theorem \ref{theorem1}.
\begin{lemma}
\label{lem:lem1}
For a calibrated predictor $\tau_n^*$ obtained using Algorithm \ref{alg:cic}, and any real-valued function $r$, we have that
\begin{equation}
\label{eq5:cal}
  \sum_{i \in \mathcal{I}_{\ell}} [r \circ \tau_n^*(W_i)]\left[\tau_n^*(W_i) - \chi_m(O_i)\right] = 0\ .
\end{equation}
\end{lemma}
 
\begin{proof}
 Note that $\tau_n^*(w)$ can be expressed pointwise for any $w \in \mathcal{W}$ as $\theta_n^* \circ \tau(w) = a_0 + \sum_{j=1}^J a_j 1(\tau(w)\geq u_j)$ for a piecewise constant function $\theta_n^*$ determined by coefficients $\{a_j\}_{j=0}^J$ and jump points $\{u_j\}_{j=1}^J$ \citep{barlow1972isotonic}. By monotonicity, we necessarily have $a_0 \in \mathbb{R}$ and $\{a_j\}_{j = 1}^{J}$ are positive coefficients. 
 
 Let $R_n(\theta):=  \sum_{i \in \mathcal{I}_{\ell}} [\theta \circ \tau(W_i) - \chi_m(O_i)]^2$ denote the least-squares risk used in the isotonic regression step. Fix an arbitrary jump point $\bar{u}_j$, and let $\xi_n : \mathbb{R}^2\rightarrow\mathbb{R}$ denote the function $\xi_n(\varepsilon, h):= \theta_n^*(h) + \varepsilon 1 (h \geq \bar{u}_j)$. Note that $\delta>0$ can be chosen to be sufficiently small that, for all $|\varepsilon|\le \delta$, $h\mapsto \xi_n(\varepsilon, h)$ is  nondecreasing --- for instance, $\delta = \min\{a_j\}_{j= 1}^{J}$ suffices. Thus, for sufficiently small $\delta > 0$, $h\mapsto \xi_n(\varepsilon, h)$ lies in the space of monotone  nondecreasing function for all $|\varepsilon|\le \delta$. In a slight abuse of notation, we let $R_n(\xi_n(\varepsilon)) :=\sum_{i \in \mathcal{I}_{\ell}}[\xi_n(\varepsilon, \tau(W_i)) - \chi_m(O_i)]^2$ and $R_n(\xi_n(-\varepsilon)) :=\sum_{i \in \mathcal{I}_{\ell}}[\xi_n(-\varepsilon, \tau(W_i)) - \chi_m(O_i)]^2$. 

Now, because $\theta_n^*$ minimizes $\theta \mapsto R_n(\theta)$ over the space of monotone  nondecreasing functions, for all $\varepsilon \geq 0$, it holds that both
$R_n(\xi_n(\varepsilon)) - R_n(\tau_n^*)\geq 0$ and $R_n(\xi_n(-\varepsilon)) - R_n(\tau_n^*)\geq 0$. Moreover, when $\varepsilon = 0$, $R_n(\xi_n(0)) - R_n(\tau_n^*)  = 0$. Therefore, if $\varepsilon$ is sufficiently close to 0, the derivative with respect to $\varepsilon$ of $R_n(\xi_n(\varepsilon)) - R_n(\tau_n^*)$ must be non-negative, and $R_n(\xi_n(-\varepsilon)) - R_n(\tau_n^*)$ must be non-positive. Hence, it must be true that
\[
\frac{d}{d\varepsilon}[R_n(\xi_n(\varepsilon)) - R_n(\theta_n^*)]\Big |_{\varepsilon = 0}\geq 0\mbox{\ \ and\ \ }
\frac{d}{d\varepsilon}[R_n(\xi_n(-\varepsilon)) - R_n(\theta_n^*)]\Big |_{\varepsilon = 0}\leq 0\ .
\]
This, in turn, implies that 
\[
2\sum_{i \in \mathcal{I}_{\ell}} 1(\tau(W_i) \geq \bar{u}_j)\left[\tau_n^*(W_i) - \chi_m(O_i)\right]\geq 0 \mbox{\ \ and\ \ }
2\sum_{i \in \mathcal{I}_{\ell}} 1(\tau(W_i) \geq \bar{u}_j)\left[\tau_n^*(W_i) - \chi_m(O_i)\right]\leq 0\ ,
\]
and so, it follows that $\sum_{i \in \mathcal{I}_{\ell}}  1(\tau(W_i) \geq \bar{u}_j)\left[\tau_n^*(W_i) - \chi_m(O_i)\right]= 0$.
Because the jump point $\bar{u}_j$ was arbitrary, we have that for all functions of the form $s(w) = b_0 + \sum_{j=1}^J b_j 1(\tau(w) \geq u_j)$ with coefficients $\{b_j\}_{j = 0}^J$, we can show that 
\begin{equation*}
\sum_{i \in \mathcal{I}_{\ell}}  s(W_i)\left[\tau_n^*(W_i) - \chi_m(O_i)\right]  = 0
\end{equation*}
by taking linear combinations of $1(\tau(w) \geq u_j)$ and noting that the score equations are linear in $s$. The main result of this lemma follows from the fact that, since both $\tau_n^*$ and $r \circ \tau_n^*$ can be expressed in this form, for any real-valued function $r$, we have that
\[
\sum_{i \in \mathcal{I}_{\ell}} r \circ \tau_n^*(W_i)\left[\tau_n^*(W_i) - \chi_m(O_i)\right] = 0\ .  
\] 
\end{proof}

\begin{lemma}
\label{lem:lem2}
Conditions \ref{assumption::A1}, \ref{assumption::A2} and \ref{assumption::A4} imply that the function classes $\mathcal{F}_{iso}$, $\mathcal{F}_{\tau,TV}$, $\mathcal{F}_{\tau,iso}$ and $\mathcal{F}_{Lip,m}$ are bounded. 
\end{lemma}
\begin{proof}
By Conditions \ref{assumption::A1}, \ref{assumption::A2} and \ref{assumption::A4}, we know that $\chi_m(o)$ is bounded uniformly over all observations $o \in \mathcal{O}$ and realizations of $\mathcal{E}_m$, that is, there exists a finite fixed constant $B$ such that $ \esssup_{m \in \mathbb{N},o \in \mathcal{O}}\chi_m(o) \leq B/2$. Hence,  as defined in the previous section, $\mathcal{F}_{iso}$ is uniformly bounded. Moreover, because $\mathcal{F}_{iso}$ is bounded, it directly implies that $\mathcal{F}_{\tau, iso}$ is bounded. Noting that functions of finite variation are bounded, in view of Condition \ref{assumption::A6}, we have that $\mathcal{F}_{TV}$ is  uniformly bounded by some constant that depends neither on $\theta$ nor $\tau$. This implies that $\mathcal{F}_{\tau,TV}$ is uniformly bounded. Finally, because $\mathcal{F}_{\tau,TV}$, $\mathcal{F}_{\tau,iso}$, $\chi_m$ and the potential outcomes are uniformly bounded, the function class $\mathcal{F}_{Lip,m}$ is also uniformly bounded.
\end{proof}

\begin{lemma}
Under conditions \ref{assumption::A6} and the conditions of Lemma \ref{lem:lem2}, the function $\tau' \mapsto E[Y_1 - Y_0 \,|\,\tau_n^*(W) = \tau']$ has total variation bounded above by three times the total variation of $\theta_0$, where $\theta_0$ is as in Condition \ref{assumption::A6}.
\label{lemma::TVnormBounded}
\end{lemma}

\begin{proof}    
 Since the function $\theta_n^*$ is nondecreasing and piecewise constant, we have $$E[Y_1 - Y_0 \,|\, (\theta_n^* \circ \tau)(W) = \tau'] = E[Y_1 - Y_0 \,|\, \tau(W) \in B_{\tau'}]$$
for the set $B_{\tau'} := \left\{ z \in \mathcal{T}: \theta_n^*(z) = \tau'\right\}$, where $B_{\tau'} = \{z \in \mathcal{T}: a(\tau') \leq z < b(\tau')\}$ for some endpoints $a(\tau'), b(\tau') \in \mathbb{R}$. The law of total expectation further implies that
$$E[Y_1 - Y_0 \,|\, \tau(W) \in B_{\tau'}] = E[\theta_0 \circ \tau (W) \,|\, \tau(W) \in B_{\tau'}]\ ,$$
where $\theta_0 $ is such that $\theta_0 \circ \tau(W)= \gamma_0(\tau,W)$ $P$-almost surely. By Condition \ref{assumption::A6}, the function $\theta_0$ is of bounded total variation. Heuristically, since $\tau' \mapsto E[\theta_0 \circ \tau (W) \,|\, \tau(W) \in B_{\tau'}]$ is obtained by locally averaging $\theta_0$ within the bins $(B_{\tau'}: \tau')$, its total variation should also be bounded. We show this formally as follows. Note first that
$$E[\theta_0 \circ \tau (W) \,|\, \tau(W) \in B_{\tau'}] = E[\theta_0^{+}  \circ \tau (W) \,|\, \tau(W) \in B_{\tau'}] - E[\theta_0^{-} \circ \tau (W) \,|\, \tau(W) \in B_{\tau'}]\ ,$$
where $\theta_0^{+}$ and $\theta_0^{-}$ are two bounded, nondecreasing functions satisfying the Jordan decomposition $\theta_0 = \theta_0^{+} - \theta_0^{-}$ (Theorem 4, Section 5.2 of \citealp{royden1963real}). Moreover, we can choose $\theta_0^{+}$ such that $\theta_0^{+}(\infty) - \theta_0^{+}(-\infty)$ is equal to the total variation of $\theta_0$. Since $\|\theta_0^{-}\|_{TV} = \|\theta_0 - \theta_0^{+}\|_{TV}\leq  \|\theta_0\|_{TV} + \|\theta_0^{+}\|_{TV}$, we have that $\|\theta_0^{-}\|_{TV}$ is bounded by $2\|\theta_0\|_{TV}$.

Since $\theta_n^*$ is nondecreasing, by definition, we have that $t_1 < t_2$ implies that $x_1 < x_2$ for any $x_1 \in B_{t_1}$ and $x_2 \in   B_{t_2}$. It follows that both $ \tau' \mapsto E[\theta_0^{+} \circ \tau (W) \,|\,\tau(W) \in B_{\tau'}]$ and $ \tau' \mapsto E[\theta_0^{-} \circ \tau (W) \,|\, \tau(W) \in B_{\tau'}]$ are nondecreasing; furthermore, they are also bounded. By Theorem 4 of \citet{royden1963real}, a function is of bounded variation if and only if it is the difference between two bounded nondecreasing functions. We conclude that $\tau' \mapsto E[Y_1 - Y_0 \,|\, \theta_n^* \circ \tau(W) = \tau'] =E[\theta_0^{+}  \circ \tau (W) \,|\, \tau(W) \in B_{\tau'}] - E[\theta_0^{-} \circ \tau (W) \,|\, \tau(W) \in B_{\tau'}]$ is of bounded variation. Moreover, its total variation norm is bounded above by the sum of the total variation norm of $E[\theta_0^{+}  \circ \tau (W) \,|\, \tau(W) \in B_{\tau'}]$ and that of $ E[  \theta_0^{-} \circ \tau (W) \,|\, \tau(W) \in B_{\tau'}]$. We recall that the total variation of monotone functions is simply the difference between the left and right endpoints of the monotone function, and that
$$\essinf_{w \in \mathcal{W}} (\theta_0^{+}  \circ \tau )(w)  \leq E[\theta_0^{+}  \circ \tau (W) \,|\, \tau(W) \in B_{\tau'}]   \leq \esssup_{w \in \mathcal{W}} (\theta_0^{+}  \circ \tau )(w), $$
and similarly for $\theta_0^{-}\circ \tau$.  As a consequence, the total variation norms of $E[\theta_0^{+}  \circ \tau (W) \,|\, \tau(W) \in B_{\tau'}] $ and $ E[  \theta_0^{-} \circ \tau (W) \,|\, \tau(W) \in B_{\tau'}] $ are bounded by the total variation norm of $\theta_0^{+} $ and that of $\theta_0^{-} $, respectively. Using the sublinearity of the total variation norm, we conclude that $\tau' \mapsto E[Y_1 - Y_0 \,|\, \theta_n^* \circ \tau(W) = \tau'] $ has total variation norm bounded above by $3 \|\theta_0\|_{TV}$.  
\end{proof}
 
\subsection{Proofs of theorems}
\subsubsection*{Proof of Theorem \ref{theorem1}}
 
\begin{proof}
Conditioning on $\mathcal{D}_n$, we have that 
\begin{align*}
&\hspace{-.4in}E\left\{\left[\gamma_0(\tau_n^*,W) - \tau_n^*(W) \right] \left[\chi_0(O)- \tau_n^*(W) \right] \,|\, \mathcal{D}_n \right\}\\
&=\
E \{ E\left\{\left[\gamma_0(\tau_n^*,W) - \tau_n^*(W) \right] \left[\chi_0(O)- \tau_n^*(W)\right] | W  \right\}\,|\, \mathcal{D}_n \}\\
&=\ E\{\left[\gamma_0(\tau_n^*,W) - \tau_n^*(W) \right] \left[\tau_0(W)- \tau_n^*(W) \right]\,|\, \mathcal{D}_n\} \\
&=\ 
E\{E\left\{\left[\gamma_0(\tau_n^*,W) - \tau_n^*(W) \right] \left[\tau_0(W)- \tau_n^*(W) \right]| \tau_n^*(W)\right\} \,|\, \mathcal{D}_n\} \\
&=\ E\{\left[\gamma_0(\tau_n^*,W) - \tau_n^*(W) \right] \left[\gamma_0(\tau_n^*,W)- \tau_n^*(W) \right]\,|\, \mathcal{D}_n\}  \\
&=\ E\{\left[\gamma_0(\tau_n^*,W) - \tau_n^*(W) \right]^2\,|\, \mathcal{D}_n\}\ . 
\end{align*}
The above equality implies that 
\begin{align}
\int \left\{ \gamma_0(\tau_n^*, w) - \tau_n^*(w) \right\}^2 dP(w)\ &=\ \int \left\{ \gamma_0(\tau_n^*, w) - \tau_n^*(w) \right\} \left\{ \chi_0(o)- \tau_n^*(w) \right\} dP(o)\notag\\
&=\ \int \left\{ \gamma_0(\tau_n^*, w) - \tau_n^*(w) \right\} \left\{ \chi_0(o)- \chi_m(o)\right\}dP(o)\label{eq1:theo1}\\
&\hspace{0.22in}+ \int \left\{ \gamma_0(\tau_n^*, w) - \tau_n^*(w) \right\} \left\{ \chi_m(o)- \tau_n^*(w) \right\} dP(o)\ .\notag
\end{align}
Note that, by Lemma \ref{lem:lem1}, for each real-valued function $r$, $\tau_n^*$ satisfies the equation
\begin{equation*}
\frac{1}{\ell}\sum_{i \in \mathcal{I}_{\ell}} r(\tau_n^*(W_i)) \left[\chi_m(O_i)- \tau_n^*(W_i)  \right] = 0\ .
\end{equation*}
Setting $r(\tau') :=  E[Y_1 - Y_0 \,|\, \tau_n^*(W) = \tau']  - \tau'$, we conclude that
$$\int \left\{ \gamma_0(\tau_n^*, w) -\tau_n^*(w) \right\} \left\{ \chi_m(o)- \tau_n^*(w) \right\} dP_{\ell}(o) = 0\ .$$
Subtracting the above score equation from the second summand in \eqref{eq1:theo1}, we obtain that
\begin{align}
 \int \left\{ \gamma_0(\tau_n^*, w) - \tau_n^*(w) \right\}^2 dP(w)\ &=\  \int \left\{ \gamma_0(\tau_n^*, w) - \tau_n^*(w) \right\} \left\{ \chi_0(o)- \chi_m(o)\right\} dP(o) \label{eq2:theo1}\\ 
&\hspace{.2in}+\int \left\{ \gamma_0(\tau_n^*, w) - \tau_n^*(w) \right\} \left\{\chi_m(o)- \tau_n^*(w) \right\} d(P-P_{\ell})(o)\nonumber\ .
\end{align}
This may be written in shorthand as $ \norm{\gamma_0(\tau_n^*, \cdot) - \tau_n^* }^2=(I)+(II)$ with \begin{align*}
(I)\ &:=\ P\{ [\gamma_0(\tau_n^*, \cdot) - \tau_n^*] (\chi_0- \chi_m)\}\\
(II)\ &:=\ (P-P_{\ell}) \{ [\gamma_0(\tau_n^*, \cdot) - \tau_n^*] (\chi_m- \tau_n^*)\}\ .
\end{align*}
In order to show the desired result, we will bound both $(I)$ and $(II)$.

We can bound $(I)$ using the law of iterated conditional expectations and the Cauchy-Schwarz inequality. First, conditioning on $\mathcal{E}_m$, we note that  
\begin{align}
\label{eq3:theo1}
 &P\{[\gamma_0(\tau_n^*, \cdot) - \tau_n^*]( \chi_0- \chi_m)\}\notag\\
 &=\ \int \left\{ \gamma_0(\tau_n^*, w) - \tau_n^*(w) \right\} E[\chi_0(O) - \chi_m(O)\,|\, W = w, \mathcal{E}_m]\,dP(w) \nonumber \\  &\leq\  \norm{ \gamma_0(\tau_n^*, \cdot) - \tau_n^*}\norm{E[\chi_0(O)\,|\, W=\,\cdot\, ]- E[\chi_m(O) \,|\, W = \, \cdot\, , \mathcal{E}_m ]}\ .
\end{align}
Next, we express the second norm in \eqref{eq3:theo1} in terms of $\norm{\pi_m - \pi_0}$ and $\norm{\mu_m- \mu_0}$. Recalling that $E[\chi_0(O) \,|\, W=w] = \tau_0(w)$, we have that
\begin{align*}
&E[\chi_m(O) \,|\, W=w, \mathcal{E}_m] - E[\chi_0(O) \,|\, W=w]\\
&=\ \mu_m(1,w) - \mu_0(1,w) - [\mu_m(0,w) - \mu_0(0,w)]+ \frac{\pi_0(w)}{\pi_m(w) } \left[\mu_0(1,w) - \mu_m(1,w) \right]\\
&\hspace{.2in}+\frac{1- \pi_0(w)}{1-\pi_m(w) } \left[\mu_0(0,w) - \mu_m(0,w) \right]\\
&=\ \left[\frac{\pi_0(w) -\pi_m(w)}{\pi_m(w) } \right]\left[\mu_0(1,w) - \mu_m(1,w)\right]+ \left[\frac{\pi_m(w) - \pi_0(w)}{1-\pi_m(w) } \right]\left[\mu_0(0,w) - \mu_m(0,w)\right].
\end{align*}
By Condition \ref{assumption::A2}, $P(1- \eta > \pi_m(W) > \eta) = 1$ for some $\eta>0$. The latter condition combined with the Cauchy-Schwarz inequality gives that  $\norm{E[\chi_0(O)\,|\, W = \,\cdot\, ]- E[\chi_m(O) \,|\, W = \,\cdot\,, \mathcal{E}_m ]}$ is bounded above by
\begin{align*}
 \norm{[\pi_m(\cdot) - \pi_0(\cdot)][\mu_0(0,\cdot) - \mu_m(0,\cdot)]}+  \norm{[\pi_m(\cdot) - \pi_0(\cdot)][\mu_0(1,\cdot) - \mu_m(1,\cdot)]}.
\end{align*}
By Condition \ref{assumption::A2}, we also have that for any $P$-measurable function $h:\mathcal{W} \rightarrow \mathbb{R}$
\begin{align*}
\int h(w)^2[\mu_0(1,w) - \mu_m(1,w)]^2 dP(w)\ &=
\ \iint h(w)^2[\mu_0(a,w) - \mu_m(a,w)]^2 \frac{a}{\pi_0(w)} P(da,dw)\\
&\leq\ \frac{1}{\eta} \iint h(w)^2[\mu_0(a,w) - \mu_m(a,w)]^2 P(da,dw)\ .
\end{align*}
The same bound holds for $\int h(w)^2 [\mu_0(0,w) - \mu_m(0,w)]^2 dP(w)$. Setting $h:w\mapsto \pi_m(w) - \pi_0(w)$, we conclude 
\begin{align}
\begin{split}
\label{eq4:theo1}
& \norm{E[\chi_m(O)\,|\, W = \,\cdot\,, \mathcal{E}_m ]- E[\chi_0(O) \,|\, W = \,\cdot\,]}
\ \lessapprox\ \norm{(\pi_m - \pi_0)(\mu_0 - \mu_m)}.
\end{split}
\end{align}
Together, \eqref{eq3:theo1} and \eqref{eq4:theo1} yield that $(I)$ is bounded above by
\begin{align}
\label{eq5:theo1}
\begin{split}
& P\{[\gamma_0(\tau_n^*, \cdot) - \tau_n^*](\chi_0- \chi_m)\} 
\ \lessapprox\ \norm{\gamma_0(\tau_n^*, \cdot) - \tau_n^* }  \norm{\left(\pi_m - \pi_0\right)\left(\mu_0 - \mu_m\right)}.    
\end{split}
\end{align}

We now find an upper bound for  $(II)$.  We claim that, conditionally on $\mathcal{E}_m$, the random functions appearing in this empirical process term are contained in fixed and uniformly bounded function classes. To see this, we note that $\tau_n^* = \theta_n^* \circ \tau$ for some $\theta_n^* \in \mathcal{F}_{iso}$ and, as a consequence, $\tau_n^* \in \mathcal{F}_{\tau,iso}$, a uniformly bounded function class by Lemma \ref{lem:lem2}, $P_0$-almost surely. By Lemma \ref{lemma::TVnormBounded}, the function $w\mapsto\gamma_0(\tau_n^*,w)$ falls in $\mathcal{F}_{\tau,TV}$. This further implies that $o \mapsto \{E[Y_1 - Y_0 \,|\, \tau_n^*(W) = \tau_n^*(w) ] - \tau_n^*(w) \} \{ \chi_m(o) - \tau_n^*(w) \} \in \mathcal{F}_{Lip,m}$, which is a uniformly bounded function class by Lemma \ref{lem:lem2}.

 Next, we let $C := \esssup_{x \in \mathcal{T}}|\theta_0(x)|$ and define $K := B + C$, where we recall that $B:=\sup_{ m \in \mathbb{N}}\sup_{\mathcal{E}_m}\esssup_{o \in \mathcal{O}}\left\{|\chi_0(o)| + |\chi_m(o)|\right\}$. Furthermore, we set $\delta_n := \norm{\gamma_0(\tau_n^*, \cdot) - \tau_n^* }$, which is a random rate. For any given rate $\delta$, we define
\begin{align*}
S_n(\delta):=&\sup_{\tau_1 \in \mathcal{F}_{\tau,TV}, \tau_2 \in \mathcal{F}_{\tau, iso}: \norm{ \tau_1 - \tau_2 }  \leq \delta } (P-P_{\ell})\{(\tau_1 - \tau_2)(\chi_m- \tau_2)\}=\sup_{f \in \mathcal{F}_{Lip,m}: \|f\| \leq \delta K}(P-P_{\ell})f\ .
\end{align*}
As a consequence of the above, we have that $(II)\leq S_n(\delta_n)$. Due to the randomness in $\delta_n$, the above cannot be further upper-bounded immediately. To bound the term above, we will take a $\delta>0$ that is deterministic conditional on $\mathcal{E}_m$, and upper-bound $\phi_n(\delta) := E \left\{S_n(\delta)\right\}$, where the expectation is also taken over $\mathcal{D}_n$. To bound the above term, we will use empirical process techniques with the function classes $\mathcal{F}_{iso}$, $\mathcal{F}_{\tau,TV}$, $\mathcal{F}_{\tau, iso}$ and $\mathcal{F}_{Lip,m}$.  To do so, we must study the uniform entropy integral \[\mathcal{J}(\delta, \mathcal{F}) :=  \int_0^{\delta} \sup_Q \sqrt{N(\varepsilon, \mathcal{F}, \norm{\,\cdot\,}_Q)}\,d\varepsilon\] for each of these function classes. By Lemma \ref{lem:lem2}, all these function classes are uniformly bounded. We note that, conditional on $\mathcal{E}_m$ so that $\chi_m$ is fixed, $\mathcal{F}_{Lip,m}$ is a multivariate Lipschitz transformation of $\mathcal{F}_{\tau,TV}$ and $\mathcal{F}_{\tau,iso}$, and therefore, by Theorem 2.10.20 of \cite{van1996weak}, we have that $\mathcal{J}(\delta, \mathcal{F}_{Lip,m}) \lessapprox  \mathcal{J}(\delta, \mathcal{F}_{\tau,TV}) + \mathcal{J}(\delta, \mathcal{F}_{\tau,iso}).$ Since functions of bounded total variation can be written as a difference of nondecreasing monotone functions, we have by the same theorem that $\mathcal{J}(\delta, \mathcal{F}_{TV}) \lessapprox \mathcal{J}(\delta, \mathcal{F}_{iso}).$
We claim the same upper bound holds up to a constant for $\mathcal{F}_{\tau,TV}$ and $\mathcal{F}_{\tau,iso}$. We establish this explcitly for $\mathcal{F}_{\tau,iso}$ below; the result for $\mathcal{F}_{\tau,TV}$ follows from an identical argument. We note that
\begin{align*}
\mathcal{J}(\delta, \mathcal{F}_{\tau,iso})& =  \int_0^{\delta} \sup_Q \sqrt{N(\varepsilon, \mathcal{F}_{\tau,iso}, \norm{\,\cdot\,}_Q)}\,d\varepsilon =  \int_0^{\delta} \sup_Q \sqrt{N(\varepsilon, \mathcal{F}_{iso}, \norm{\,\cdot\,}_{Q \circ \tau^{-1}})}\, d\varepsilon  = \mathcal{J}(\delta, \mathcal{F}_{iso})\ ,  
\end{align*}
where $Q \circ \tau^{-1}$ is the push-forward probability measure for the random variable $\tau(W)$.  We now proceed with bounding $\phi_n(\delta)$. Applying Theorem 2.10.20 of \cite{van1996weak}, we obtain, for any $\delta > 0$ deterministic conditionally on $\mathcal{E}_m$, that
\begin{align}
\label{eq7:theo1}
E\left[S_n(\delta)\,|\, \mathcal{E}_m\right]\ &\lessapprox\  \ell^{-1/2} \mathcal{J}(\delta, \mathcal{F}_{Lip,m}) \left(1 + \frac{\mathcal{J}(\delta, \mathcal{F}_{Lip,m})}{\sqrt{\ell} \delta^2} \right)\nonumber  \\
&\lessapprox\ \ell^{-1/2}\mathcal{J}(\delta, \mathcal{F}_{iso}) \left(1 + \frac{\mathcal{J}(\delta, \mathcal{F}_{iso})}{\sqrt{\ell} \delta^2} \right),
\end{align}
where the right-hand side can only be random through $\delta$. 

We can now proceed with the main argument that gives a rate of convergence for $\delta_n$. First, we note that combining Equations \ref{eq2:theo1} and \ref{eq5:theo1} yields that the event 
\begin{align*}
\left\{\norm{\gamma_0(\tau_n^*, \cdot) - \tau_n^* }^2 \leq \norm{\gamma_0(\tau_n^*, \cdot) - \tau_n^* }\norm{(\pi_m - \pi_0)(\mu_m - \mu_0)} + S_n(\delta_n)\right\}
\end{align*}
occurs with probability one. We then proceed with a peeling argument to account for the randomness of  $\delta_n$. Let $\varepsilon_n$ be any given sequence that is deterministic conditional on $\mathcal{E}_m$, and define $A_s$ as the event $\left\{2^{s+1} \varepsilon_n \geq \norm{\gamma_0(\tau_n^*, \cdot) - \tau_n^* } \geq 2^s \varepsilon_n\right\}$ as well as the random quantity $\epsilon_m^{nuis} := \norm{(\pi_m - \pi_0)(\mu_m - \mu_0)}$. Then, for any $S>0$, we have that
\begin{align}
\label{eq8:theo1}
\left(\norm{\gamma_0(\tau_n^*, \cdot) - \tau_n^* } \geq 2^S \varepsilon_n\right)\
 &=\ \sum_{s=S}^\infty P \left(2^{s+1} \varepsilon_n \geq \norm{\gamma_0(\tau_n^*, \cdot) - \tau_n^* } \geq 2^s \varepsilon_n\right)\ =\ \sum_{s=S}^\infty P(A_s) \nonumber \\
&= \sum_{s=S}^\infty P \Bigg(A_s, \delta_n^2
\leq \delta_n\epsilon_m^{nuis}+ S_n(\delta_n) \Bigg)\ .
\end{align}
In all the events in the above sum, we have that $S_n(\delta_n) \leq S_n(2^{s+1}\varepsilon_n)$ since $\delta_n = \norm{\gamma_0(\tau_n^*, \cdot) - \tau_n^*}$. Next, manipulating the inequalities in the above events, we have that
\begin{align*}
\left\{A_s, \delta_n^2 \leq  \delta_n\epsilon_m^{nuis} + S_n(\delta_n) \right\}\ &\subseteq\ \left\{A_s, \delta_n^2 \leq  2^{s+1}\varepsilon_n \epsilon_m^{nuis} + S_n(2^{s+1}\varepsilon_n) \right\}  \\
&\subseteq\ \left\{2^{2s} \varepsilon_n^2\leq  \delta_n^2 \leq 2^{s+1}\varepsilon_n\epsilon_m^{nuis} + S_n(2^{s+1}\varepsilon_n) \right\}  \\
&\subseteq\ \left\{2^{2s} \varepsilon_n^2\leq 2^{s+1}\varepsilon_n\epsilon_m^{nuis} + S_n(2^{s+1}\varepsilon_n) \right\},
\end{align*}
which implies that the sum in \eqref{eq8:theo1} is upper bounded by 
\begin{align*}
\sum_{s=S}^\infty P \bigg( &2^{2s}\varepsilon_n^2 \leq 2^{s+1} \varepsilon_n \epsilon_m^{nuis} + S_n(2^{s+1}\varepsilon_n)  \bigg)\ . 
\end{align*}

Using \eqref{eq7:theo1} and Markov's inequality, we find that
\begin{align*}
&\sum_{s=S}^\infty P \bigg( 2^{2s}\varepsilon_n^2 \leq 2^{s+1} \varepsilon_n\epsilon_m^{nuis}+ S_n(2^{s+1}\varepsilon_n)  \bigg) \\   
&\leq\ \sum_{s=S}^\infty   E \left\{P \bigg( 2^{2s}\varepsilon_n^2 \leq 2^{s+1} \varepsilon_n\epsilon_m^{nuis}+ S_n(2^{s+1}\varepsilon_n)    \,|\, \mathcal{E}_m  \bigg)\right\} \\   
 &\leq\ \sum_{s=S}^\infty   E \left\{
 \frac{2^{s+1} \varepsilon_n \epsilon_m^{nuis} + E[S_n(2^{s+1}\varepsilon_n) \,|\, \mathcal{E}_m] }{2^{2s}\varepsilon_n^2} \right\} \\
 &\lessapprox\ \sum_{s=S}^\infty  E \left[
 \frac{ \epsilon_m^{nuis}}{2^{s-1}\varepsilon_n} +\frac{ \mathcal{J}(2^{s+1}\varepsilon_n, \mathcal{F}_{iso})}{2^{2s}\sqrt{\ell}\varepsilon_n^2}\left(1 + \frac{\mathcal{J}(2^{s+1}\varepsilon_n, \mathcal{F}_{iso})}{\sqrt{\ell}2^{2s+1}\varepsilon_n^2} \right) \right] .
\end{align*}

As a consequence of Lemma \ref{lem:lem2} and the covering number bound for bounded monotone functions given in Theorem 2.7.5 of \citet{van1996weak}, we have that $\mathcal{J}( 2^{s+1}\varepsilon_n, \mathcal{F}_{iso}) = 2^{s/2 + 1/2} \sqrt{\varepsilon_n}$. Using this fact, we find that 
$$\frac{ \mathcal{J}(2^{s+1}\varepsilon_n, \mathcal{F}_{iso}) }{2^{2s}\sqrt{\ell}\varepsilon_n^2} \lessapprox\frac{1}{2^s}\frac{ \mathcal{J}( \varepsilon_n, \mathcal{F}_{iso}) }{ \sqrt{\ell}\varepsilon_n^2}\ ,$$
from which it follows that 
\begin{align*}
& \frac{ \mathcal{J}(2^{s+1}\varepsilon_n, \mathcal{F}_{iso}) \left(1 + \frac{\mathcal{J}(2^{s+1}\varepsilon_n, \mathcal{F}_{iso})}{\sqrt{\ell}2^{2s+1}\varepsilon_n^2} \right) }{2^{2s}\sqrt{\ell}\varepsilon_n^2} 
\lessapprox2^{-s}\frac{ \mathcal{J}(\varepsilon_n, \mathcal{F}_{iso}) \left(1 + \frac{\mathcal{J}( \varepsilon_n, \mathcal{F}_{iso})}{\sqrt{\ell} \varepsilon_n^2} \right) }{ \sqrt{\ell}\varepsilon_n^2}\ .
\end{align*}
We now choose $\varepsilon_n := \max\{\ell^{-1/3}, \norm{(\pi_m - \pi_0)(\mu_m - \mu_0)}\}$, which indeed is  deterministic conditional on $\mathcal{E}_m$. This choice ensures that $\mathcal{J}(\varepsilon_n, \mathcal{F}_{iso}) \lessapprox \sqrt{\ell} \varepsilon_n^2$ and $\epsilon_m^{nuis} = \norm{(\pi_m - \pi_0)(\mu_m - \mu_0)} \lessapprox \varepsilon_n$, so that 
$$ 
 \frac{ \epsilon_m^{nuis}}{2^{s-1}\varepsilon_n} +\frac{ \mathcal{J}(2^{s+1}\varepsilon_n, \mathcal{F}_{iso}) }{2^{2s}\sqrt{\ell}\varepsilon_n^2} \left(1 + \frac{\mathcal{J}(2^{s+1}\varepsilon_n, \mathcal{F}_{iso})}{\sqrt{\ell}2^{2s+1}\varepsilon_n^2} \right) \lessapprox \frac{ 1}{ 2^{s}}\ ,$$
 where the right-hand side is nonrandom. Thus, we have that
\begin{align*}
P \left(\norm{\gamma_0(\tau_n^*, \cdot) - \tau_n^* } \geq 2^S \varepsilon_n\right) \ \lessapprox\ \sum_{s=S}^\infty \frac{ 1}{ 2^{s}} \xrightarrow[S \to\infty]{}   0\ .
\end{align*}
As a consequence, for every $\varepsilon > 0$, we can find a constant $2^S$ sufficiently large such that $P \left(\norm{\gamma_0(\tau_n^*, \cdot) - \tau_n^* } \geq 2^S \varepsilon_n\right) < \varepsilon$. In other words, we have shown that $\norm{\gamma_0(\tau_n^*, \cdot) - \tau_n^* } = O_P(\varepsilon_n)$ for our choice of $\varepsilon_n$, and so,  $CAL(\tau_n^*)=\norm{\gamma_0(\tau_n^*, \cdot) - \tau_n^* }^2=O_P(\varepsilon_n^2)$. The result follows from that the fact that $\varepsilon_n^2\leq \ell^{-2/3}+ \norm{(\pi_m - \pi_0)(\mu_m - \mu_0)}^2$.
\end{proof}

\subsubsection*{Proof of Theorem \ref{theoremmedian}}
\begin{proof}
By the definition of the pointwise median stated in Section \ref{sec::notation},  for each covariate value $w \in \mathcal{W}$, there exists some random index $j_n(w)$ such that $\tau_n^*(w) = \tau_{n,j_n(w)}^*(w)$. (We note here that this property may fail for other definitions of the median when $k$ is even.) Thus, we have that $|\gamma_0(\tau_n^*, w) - \tau_n^*(w)| = |\gamma_0(\tau_{n,j_n(w)}^*, w) - \tau_{n,j_n(w)}^*(w)| \leq \sum_{s=1}^k |  \gamma_0(\tau_{n,s}^*, w) - \tau_{n,s}^*(w) |$, and so,
\begin{align*}
   \norm{\gamma_0(\tau_n^*, \cdot) - \tau_n^*}\ \leq\ \norm{\sum_{s=1}^k |  \gamma_0(\tau_{n,s}^*, \cdot) - \tau_{n,s}^* | }\ &\leq\ \sum_{s=1}^k \norm{   \gamma_0(\tau_{n,s}^*, \cdot) - \tau_{n,s}^* }\\
   &\leq\ \sqrt{k\sum_{s=1}^k \norm{   \gamma_0(\tau_{n,s}^*, \cdot) - \tau_{n,s}^* }^2}\ ,
\end{align*}
where the final inequality follows from the Cauchy-Schwarz inequality. Squaring both sides gives that  $\normalfont\text{CAL}( \tau_n^*) \leq k\sum_{s=1}^k \normalfont\text{CAL}( \tau_{n,s}^*)$, as desired.
\end{proof}

\subsubsection*{Proof of Theorem \ref{theorem2}}
\begin{proof}
As before, we may write $\tau_n^* = \theta_n^* \circ \tau$ for some $\theta_n^* \in \mathcal{F}_{iso}$ that minimizes the empirical risk
\begin{align*}
R_n(\theta): \theta \mapsto   \sum_{i \in \mathcal{I}_{\ell} } \left[\chi_m(O_i) - \theta \circ \tau(W_i)\right]^2
\end{align*}over $\mathcal{F}_{iso}$.
For any given $\theta\in\mathcal{F}_{iso}$, the one-sided path $\{\varepsilon \mapsto \theta_n^* + \varepsilon ( \theta - \theta_n^*): \varepsilon \in [0,1]\}$ through $\theta_n^*$ lies entirely in $\mathcal{F}_{iso}$ since $\mathcal{F}_{iso}$ is a convex space. Furthermore, we have that
\begin{align}
\begin{split}
  &  - 2   \sum_{i \in \mathcal{I}_{\ell} }(\theta - \theta_n^* ) \circ \tau(W_i)[\chi_m(O_i) - \theta_n^* 
\circ \tau(W_i)]
=\lim_{\varepsilon \downarrow 0} \frac{R_n(\theta_n^* + \varepsilon ( \theta - \theta_n^*)) - R_n(\theta_n^*)}{\varepsilon}\geq 0
\end{split}
\label{proofeqn::ERMscoreEqn}
\end{align}
for all $\theta \in \mathcal{F}_{iso}$. The oracle isotonic risk minimizer $\tau_0^*$ can be expressed as $\tau_0^* = \theta_0 \circ \tau$ where $\theta_0 := \argmin_{\theta \in \mathcal{F}_{iso}} \norm{\theta \circ \tau - \tau_0}$. Taking $\theta = \theta_0$ in (\ref{proofeqn::ERMscoreEqn}), we obtain the inequality
\begin{equation}
  \sum_{i \in \mathcal{I}_{\ell} } [(\theta_0 - \theta_n^* ) \circ \tau(W_i)][\chi_m(O_i) - \theta_n^* 
\circ \tau(W_i)]  \leq 0\ .
\label{eqn::score}
\end{equation}
Rearranging terms and adding and subtracting $P_{\ell}\{[(\theta_0 - \theta_n^* ) \circ \tau](\chi_0)\}$ in the above inequality implies that
$P_{\ell} \{[(\theta_0 - \theta_n^* ) \circ \tau](\chi_m  - \chi_0)\} \leq  P_{\ell} \{[(\theta_0 - \theta_n^* ) \circ \tau](\theta_n^* 
\circ \tau - \chi_0 )\}$.
Adding and subtracting $ P\{[(\theta_0 - \theta_n^* ) \circ \tau](\theta_n^* 
\circ \tau - \chi_0)\}$ yields that
\begin{align}
\label{proofeqn::Theorem2tmp1}
&P_{\ell}\{[(\theta_0 - \theta_n^* ) \circ \tau](\chi_m  - \chi_0)\} \nonumber  - (P_{\ell} - P)\{[(\theta_0 - \theta_n^* ) \circ \tau](\theta_n^* 
\circ \tau - \chi_0)\}\nonumber\\
& \leq\ P\{[(\theta_0 - \theta_n^* ) \circ \tau](\theta_n^* 
\circ \tau - \chi_0)\}\ . 
\end{align}
Next, adding and subtracting $P\{(\theta_0\circ \tau) [(\theta_0 -\theta_n^*)\circ\tau]\}$, we have that
\begin{align}
\label{proofeqn::tmp3}
&P\{[(\theta_0 - \theta_n^* ) \circ \tau](\theta_n^* 
\circ \tau - \chi_0)\}\nonumber
\\
&=\ P\{[(\theta_0 - \theta_n^* ) \circ \tau][\theta_n^* 
\circ \tau - E[\chi_0(O) \,|\, W = \cdot\,]] \} \nonumber \\ 
&=\ P\{[(\theta_0 - \theta_n^* ) \circ \tau](\theta_n^* 
\circ \tau - \tau_0)\}  \nonumber \\
&=\ P\{[(\theta_0 - \theta_n^* ) \circ \tau ][(\theta_n^* -\theta_0)
\circ \tau]\}+ P\{[(\theta_0 - \theta_n^* ) \circ \tau](\theta_0
\circ \tau - \tau_0)\} \nonumber\\
&=\ P\{[(\theta_0 - \theta_n^* ) \circ \tau](\theta_0
\circ \tau - \tau_0)\}-\|(\theta_0 - \theta_n^* ) \circ \tau \|^2\ ,
\end{align}
where we used the fact that $E[\chi_0(O) \,|\, W=w] = \tau_0(w)$. Next, we note that $\theta_0$ minimizes the population risk function  $\theta \mapsto E_{P}[\tau_0(W) - \theta \circ \tau(W)]^2$ over $\mathcal{F}_{iso}$. As a consequence, the same argument used to derive \eqref{eqn::score} can be used to obtain that
$P\{[(\theta - \theta_0 ) \circ \tau](\tau_0 - \theta_0
\circ \tau)\}\leq 0$
for any $\theta \in \mathcal{F}_{iso}$. Taking $\theta = \theta_n^*$, we find that
\begin{equation}
    P\{[(  \theta_0-\theta_n^* ) \circ \tau](\theta_0
\circ \tau - \tau_0)\}  \leq 0\ .
\label{proofeqn::tmp2}
\end{equation}
Combining \eqref{proofeqn::tmp3} and \eqref{proofeqn::tmp2}, we obtain that
\begin{equation}
    P \{[(\theta_0 - \theta_n^* ) \circ \tau](\theta_n^* 
\circ \tau - \chi_0)\} \leq -\norm{(\theta_0 - \theta_n^* ) \circ \tau }^2.  \label{proofeqn::tmp5}
\end{equation}
Finally, combining \eqref{proofeqn::Theorem2tmp1} and \eqref{proofeqn::tmp5}, we obtain the following inequality
\begin{align*}
\norm{(\theta_0 - \theta_n^* ) \circ \tau }^2 \leq - P_{\ell}\{[(\theta_0 - \theta_n^* ) \circ \tau](\chi_m  - \chi_0)\} + (P_{\ell} - P)\{[(\theta_0 - \theta_n^* ) \circ \tau](\theta_n^* 
\circ \tau - \chi_0)\}\ .
\end{align*}
Adding and subtracting $P\{[(\theta_0 - \theta_n^* ) \circ \tau](\chi_m  - \chi_0)\}$ and noting that $ \tau_0^* - \tau_n^*  =  (\theta_0 - \theta_n^* )\circ \tau$, we finally obtain the key inequality
\begin{align}
\label{eq1:theo2}
     \norm{\tau_0^* - \tau_n^*}^2 \ &\leq\ P[(\tau_0^* - \tau_n^*)(\chi_0  - \chi_m)]+ (P-P_{\ell})[(\tau_0^* - \tau_n^*)(\chi_m  - \chi_0)] \nonumber\\
     &\hspace{.2in}+ (P_{\ell} - P)[(\tau_0^* - \tau_n^*)(\tau_n^* - \chi_0)]\ .
\end{align}
The above is similar to \eqref{eq2:theo1} in the proof of Theorem  \ref{theorem1}, and a similar proof technique is used to establish a convergence rate for $\tau_n^*$. Specifically, we use the Cauchy-Schwarz inequality to bound the first term on the right-hand side of \eqref{eq1:theo2} in terms of $\norm{\tau_0^* - \tau_n^*}$, and empirical process techniques to bound the remaining terms in terms of a function of $\norm{\tau_0^* - \tau_n^*}$ with high probability. Using a similar approach as for the derivation of \eqref{eq5:theo1}, we can upper-bound the first term of the right-hand side of \eqref{eq1:theo2} as
$P[(\tau_0^* - \tau_n^*)(\chi_0-\chi_m)]
\leq  \|\tau_0^* - \tau_n^*\|\|(\pi_m - \pi_0)(\mu_m - \mu_0)\|$. The second term in the right-hand side of \eqref{eq1:theo2} can be examined as follows. We let 
$\mathcal{F}_{4,m}:= \{(\tau_1 - \tau_2)(\chi_m-\chi_0); \tau_1, \tau_2 \in \mathcal{F}_{\tau,iso}\}$, and define $Q:= \sup_{o \in \mathcal{O}}\chi_0(o)$, which is finite in view of Conditions \ref{assumption::A1} and \ref{assumption::A2}. Additionally, we let $R:= Q + B$, and define for any fixed $\delta \in \mathbb{R}$
\[
Z_{1,n}(\delta):=\sup_{\theta_1, \theta_2 \in \mathcal{F}_{iso}: \norm{(\theta_1- \theta_2) \circ \tau} \leq \delta R }(P-P_{\ell})\{[ (\theta_1 - \theta_2 ) \circ \tau](\chi_m- \chi_0) \}=\sup_{f \in \mathcal{F}_{4,m}: \norm{f} \leq \delta R}(P-P_{\ell})f.\]
Letting $\delta_{1,n}:= \norm{\tau_0^* - \tau_n^*}$, we have that 
$(P-P_{\ell})[(\tau_0^* - \tau_n^*)(\chi_m  - \chi_0)] \leq Z_{1,n}(\delta_{1,n})$.
We
note that $\mathcal{F}_{4,m}$ is a Lipschitz transformation of the function classes $\mathcal{F}_{\tau,iso}$ and $\mathcal{F}_{\tau,iso}$, and so, for every $\delta > 0$ that is deterministic conditional on $\mathcal{E}_m$, we have that 
\begin{align*}
\psi_{1,n}(\delta \mid \mathcal{E}_m) := E [ Z_{1,n}(\delta) \mid \mathcal{E}_m]\ \lessapprox\  \ell^{-1/2}\mathcal{J}(\delta, \mathcal{F}_{iso}) \left(1 + \frac{\mathcal{J}(\delta, \mathcal{F}_{iso})}{\sqrt{\ell} \delta^2} \right)
\end{align*} in view of Theorem 2.10.20 of \cite{van1996weak} and the results outlined in Theorem \ref{theorem1}, 
where the right-hand side can only be random through $\delta$.
Finally, the third term in \eqref{eq1:theo2} can be studied as follows.  We let $\mathcal{F}_{5}:= \{(\tau_1-\tau_2)(\tau_2 - \chi_0): \tau_1, \tau_2 \in \mathcal{F}_{\tau,iso}\}$, and for any given $\delta>0$, we define
\[
Z_{2,n}(\delta) :=\sup_{\theta_1, \theta_2 \in \mathcal{F}_{iso}: \norm{(\theta_1- \theta_2) \circ \tau} \leq \delta G }(P-P_{\ell}) \{[(\theta_1 - \theta_2 ) \circ \tau](\theta_2 - \chi_0)\}= \sup_{f \in \mathcal{F}_{5}: \norm{f} \leq \delta G}(P-P_{\ell})f
\]
with $G:=Q+B$. We note that $\mathcal{F}_5$ is a Lipschitz transformation of $\mathcal{F}_{\tau,iso}$. Hence, similarly as above, for any $\delta > 0$ that is nonrandom conditional on $\mathcal{E}_m$, we have  that 
\[
\psi_{2,n}(\delta \mid \mathcal{E}_m):= E[Z_{2,n}(\delta) \mid \mathcal{E}_m]\ \lessapprox\ \ell^{-1/2}\mathcal{J}(\delta, \mathcal{F}_{iso}) \left(1 + \frac{\mathcal{J}(\delta, \mathcal{F}_{iso})}{\sqrt{\ell} \delta^2} \right),
\]
where the right-hand side can only berandom through $\delta$.
Defining $\epsilon_{m}^{nuis}:= \norm{(\pi_m - \pi_0)(\mu_m - \mu_0)}$, by a similar peeling argument as in Theorem \ref{theorem1}, for any rate $\varepsilon_n$ that is nonrandom conditional on $ \mathcal{E}_m$, we can show that
\begin{align*}
P \left(\norm{\tau_0^* - \tau_n^* } \geq 2^S \varepsilon_n\right)\ &\leq\   \sum_{s=S}^\infty E\left[
 \frac{2^{s+1} \varepsilon_n \epsilon_{m}^{nuis} + \psi_{1,n}(2^{s+1}\varepsilon_n \mid \mathcal{E}_m) +\psi_{2,n}(2^{s+1}\varepsilon_n \mid   \mathcal{E}_m) }{2^{2s}\varepsilon_n^2} \right]\\
& \lessapprox\  \sum_{s=S}^\infty E \left[
 \frac{ \epsilon_{m}^{nuis}}{2^{s-1}\varepsilon_n} +\frac{ \mathcal{J}(2^{s+1}\varepsilon_n, \mathcal{F}_{iso})}{2^{2s}\sqrt{\ell}\varepsilon_n^2}\left(1 + \frac{\mathcal{J}(2^{s+1}\varepsilon_n, \mathcal{F}_{iso})}{\sqrt{\ell}2^{2s+1}\varepsilon_n^2} \right)
\right].
\end{align*}
Then, by the same arguments used in Theorem \ref{theorem1} and the same choice of $\mathcal{E}_m$-random $\varepsilon_n$, we can establish that
$\norm{\tau_0^* - \tau_n^*} = O_P( \ell^{-1/3}) 
+ O_P(\norm{(\pi_m - \pi_0)(\mu_m-\mu_0)})$.
By the triangle inequality and the fact that $\tau_0^* = \argmin_{\theta \circ \tau: \theta \in \mathcal{F}_{iso}} \norm{\tau_0 - \theta \circ \tau}$ implies $ \norm{ \tau_0 - \tau_0^*} \leq \norm{ \tau_0 - \tau} $, we find that
$\norm{ \tau_0 - \tau_n^*}\leq  \norm{ \tau_0 - \tau_0^*} + \norm{\tau_0^* - \tau_n^*}  \leq \norm{ \tau_0 - \tau} + \norm{\tau_0^* - \tau_n^*}$.
Combining these bounds, we find that $\norm{ \tau_0 - \tau_n^*} \leq \norm{\tau_0 - \tau} +  O_P( \ell^{-1/3}) + O_P(\norm{(\pi_m - \pi_0)(\mu_m-\mu_0)})$.
\end{proof}

\subsection{Statement and proof of generalized Theorem \ref{theorem1} for random predictor}
\label{appendix::theorem1Generaliz}
Here, we consider the same setup as Theorem \ref{theorem1} but allow $\tau_n^*$ to be obtained from a random predictor $\tau_m$, as  long as $\tau_m$ is built using only data in $\mathcal{E}_m$.

\begin{condition}[independence of predictor]
\label{assumption::indOfPred}
The predictor $w\mapsto\tau_m(w)$ is independent of $\mathcal{C}_{\ell}$.
\end{condition}

\begin{theorem}[Calibration with random predictors]
\label{theorem1::RandomPredictor}
Provided Conditions \ref{assumption::A1}--\ref{assumption::indOfPred} hold, it holds that
\[\normalfont\text{CAL}(\tau_n^*)= O_P\left(\ell^{-2/3}  + \norm{(\pi_m - \pi_0)(\mu_m - \mu_0)}^2   \right).\]
\end{theorem}

\begin{proof}
Arguing exactly as in Theorem \ref{theorem1} with $\tau$ taken to be $\tau_m$ and conditioning on $\mathcal{E}_m$ as needed, we obtain the basic inequality stating that
\[
 \norm{\gamma_0(\tau_n^*, \cdot) - \tau_n^* }^2\leq P\{[\gamma_0(\tau_n^*, \cdot) - \tau_n^*](\chi_0- \chi_m)\}+(P-P_{\ell})\{ [\gamma_0(\tau_n^*, \cdot) - \tau_n^*](\chi_m- \tau_n^*)\}
\]
$P$-almost surely, where $\tau_n^* := \theta_n^* \circ \tau_m$. To establish the result of the theorem, we only need to make minor modifications to the proof of Theorem \ref{theorem1} to allow $\tau$ to be replaced by $\tau_m$. We sketch those modifications here. A core component of the proof of Theorem \ref{theorem1} involved upper-bounding $ E[S_n(\delta)\,|\, \mathcal{E}_m]$; this must now be done with $S_n(\delta)$ defined as
\begin{align*}
\sup_{\tau_1 \in \mathcal{F}_{\tau_m,TV}, \tau_2 \in \mathcal{F}_{\tau_m, iso}: \norm{ \tau_1 - \tau_2  }\leq \delta } (P-P_{\ell}) [(\tau_1 - \tau_2)(\chi_m- \tau_2)]=\sup_{f \in \mathcal{F}_{Lip,m}: \norm{f} \leq \delta K}(P-P_{\ell})f
\end{align*}
with $\tau_m$ now a random predictor. Previously, we showed that $ E[S_n(\delta) \,|\, \mathcal{E}_m]$ can be bounded by a nonrandom constant depending on $n$, $m$ and $\delta$ that is independent of $\mathcal{E}_m$. To do so, we showed that the random function class $ \mathcal{F}_{Lip,m}$ is fixed conditional on $\mathcal{E}_m$, uniformly bounded, and has uniform entropy integral bounded by the uniform entropy integral of $\mathcal{F}_{iso}$. It suffices to show that this remains true when $\tau$ is replaced by $\tau_m$. Since $\tau_m$ is obtained from $\mathcal{E}_m$, as with $\chi_m$, the predictor $\tau_m$ is deterministic conditionally on $\mathcal{E}_m$. As a consequence, the function classes $\mathcal{F}_{\tau_m, TV}$ and $\mathcal{F}_{\tau_m, iso}$, which are now random through $\tau_m$, are fixed conditional on $\mathcal{E}_m$. Since $\mathcal{F}_{Lip,m}$ is obtained from a Lipschitz transformation of elements of $\mathcal{F}_{\tau_m, TV}$ and $\mathcal{F}_{\tau_m, iso}$, we have that $\mathcal{F}_{Lip,m}$ is also fixed conditional on $\mathcal{E}_m$. Moreover, by the same argument as in the proof of Lemma \ref{lem:lem2}, which also holds for random $\tau$, these function classes are uniformly bounded by a nonrandom constant almost surely. Finally, the preservation of the uniform entropy integral argument of the proof of Theorem \ref{theorem1} is valid with $\tau$ random. With these modifications to the proof of Theorem \ref{theorem1}, the result follows. 
\end{proof}
 
\section{Simulation studies}\label{subsecApp:simstudy}

\subsection{Data-generating mechanisms}\label{subsecApp:simdatainfo}

In simulation studies, data units were generated as follows for the two scenarios considered.\\

\noindent Scenario 1:
\begin{enumerate}
    \item generate $W_1,W_2,\ldots,W_4$ independently from the uniform distribution on $(-1,+1)$;
    \item given $(W_1,W_2,W_3,W_4)=(w_1,w_2,w_3,w_4)$, generate $A$ as a Bernoulli random variable with success probability $\pi_0(w_1,w_2,w_3,w_4):=\text{expit}\{-0.25 -w_1 + 0.5w_2 - w_3 + 0.5 w_4\}$;
    \item given $(W_1,W_2,W_3,W_4)=(w_1,w_2,w_3,w_4)$ and $A=a$, generate $Y$ as a Bernoulli random variable with success probability
    	 $\mu_0(a,w_1,w_2,\ldots,w_4):=\text{expit}\{1.5 + 1.5  a + 2  a  |w_1|  |w_2| - 2.5  (1-a)  |w_2|  w_3 + 2.5  w_3 + 2.5  (1-a)  \sqrt{|w_4|} - 1.5  a  I(w_2 < 0.5) + 1.5  (1-a)  I(w_4 < 0)\}$.
\end{enumerate}

\noindent Scenario 2:
\begin{itemize}
    \item generate $W_1,W_2,\ldots,W_{20}$ independently from the uniform distribution on $(-1,+1)$;
    \item given $(W_1,W_2,\ldots,W_{20})=(w_1,w_2,\ldots,w_{20})$, generate $A$ as a Bernoulli random variable with success probability $\pi_0(w_1,w_2,\ldots,w_{20}):=\text{expit}\{0.2-0.5w_1-0.5w_2-0.5w_3+0.5w_4-0.5w_5+0.5w_6-0.5w_7-0.5w_8-0.5w_9-0.2w_{10}+0.5w_{11}-w_{12}+w_{13}-1.5w_{14}+w_{15}-w_{16}+2w_{17}-w_{18}+1.5w_{19}-w_{20}\}$;
    \item given $(W_1,W_2,\ldots,W_{20})=(w_1,w_2,\ldots,w_{20})$ and $A=a$, generate $Y$ as a normal random variable with mean $\mu_0(a,w_1,w_2,\ldots,w_{20})=-0.5 + 3.5  a + 3  a  w_1 + 6.5  (1-a)  w_2 + 1.5  a  w_3 + 4  (1-a)  w_4 + 2.5  a  w_5 - 6  (1-a)  w_6 + 1  a  w_7 + 4.5  (1-a)  w_8 + a  w_9  + 2.5  (1-a)  w_{10} +  1.5  w_{11} - 2.5   w_{12} + w_{13} - 1.5  w_{14} + 3  w_{15} - 2  w{_16} + 3  w_{17} - w_{18} + 1.5 w_{19} - 2  w_{20}$ and unit variance.
\end{itemize}
Coefficients of the propensity score logistic regression models above were selected such that the probabilities of treatment were bounded between 0.05 and 0.95 in the low-dimensional case (Scenario 1), and between 0.01 and 0.99 in the high-dimensional setting (Scenario 2).

\subsection{Implementation of the causal isotonic calibrator}\label{subsecApp:impcic}

In our simulation studies, we followed Algorithm \ref{alg:ciccrossfit::pooled} to fit the causal isotonic calibrator. In particular, we estimated the components of $\chi_0$ (i.e., $\mu_0$ and $\pi_0$) using the Super Learner \citep{van2007super} in Scenario 1, and penalized regression in Scenario 2. Super learner is an ensemble learning approach that uses cross-validation to select a convex combination of a library of candidate prediction methods. Table \ref{table1:app} shows the library of prediction models we used to estimate $\mu_0$ and $\pi_0$. Note that all of our models for the outcome regression were misspecified in Scenario 1 because of the nonlinearities in the true outcome regression. However, in both scenarios, the propensity score estimator was a consistent estimator of the true propensity score. Additionally, for numerical stability, we imposed a threshold on the estimated propensity scores such that it took values between 0.01 and 0.99. We used the R package \textsf{sl3}  \citep{coyle2021sl3-rpkg} to implement the estimation procedure. Finally, we used the R function \textsf{isoreg} to 
performed the isotonic regression step. 

\begin{center} 
\begin{table}[H]
  \caption{Information on the set of estimators used by the Super Learner to estimate the pseudo-outcome components. Abbreviations:  generalized additive models (GAM), generalized linear model (GLM), generalized linear model with lasso regularization (GLMnet), gradient boosted trees (GBRT), random forests (RF), multivariate adaptive regression splines (MARS).}
\label{table1:app}
\vskip 0.15in
\centering
\begin{tabular}{c|cc}
scenario & library for $\mu_0$  &
  library for $\pi_0$ \\ \hline
1 & logistic regression, GLMnet, GAM, & logistic regression, GLMnet, GAM,\\
  & GBRT with depth $\in\{2,3,5,6,8\}$, & GBRT with depth $\in\{2,4,6\}$ \\
 & RF, MARS & \\
 \hline
2 & GLMnet & GLMnet
\end{tabular}
\end{table}
\end{center}

\subsection{Performance metrics}\label{subsecApp:perfmet}

We estimated the performance metrics as follows. With a slight abuse of notation, let $\hat{\tau}$ denote an arbitrary estimated treatment effect predictor or its calibrated version. For each fitted $\hat{\tau}$ in a given simulation, we computed its mean squared error by taking the empirical mean of the squared difference between the fitted values of the CATE estimator and $\tau_0$,
\begin{equation*}
\widehat{\text{MSE}}(\hat{\tau}):= \frac{1}{n_{\mathcal{V}}}\sum_{i:w_i \in \mathcal{V}}[\hat{\tau}(w_i) - \tau_0(w_i)]^2.
\end{equation*}
We obtained the estimated calibration measure in two steps. We recall that the calibration measure for a given predictor $\tau$ is 
\begin{equation*}
 \int\left[\gamma_0(\tau,w) - \tau(w)\right]^2dP_W(w)\ .
\end{equation*}
First, we estimated $\gamma_0(\hat{\tau},w)$ using an independent dataset of 100,000 observations and fitted gradient boosted regression trees with the fitted values of the treatment effect predictors as covariates and the true CATE as outcome. For each simulation setting and CATE estimator, the depths of each of the regression trees were obtained using cross-validation in a separate simulation. Let $\hat{\gamma}_0(\hat{\tau},w)$ denote the estimated function. In the second step, we used the sample $\mathcal{V}$ to estimate the calibration measure as
\begin{equation*}
\widehat{\text{CAL}}(\tau):= \frac{1}{n_{\mathcal{V}}}\sum_{i: w_i \in \mathcal{V}}\left[\tau_0(w_i) -  \hat{\tau}(w_i)\right]\left[\hat{\gamma}_0(\hat{\tau},w_i) - \hat{\tau}(w_i)\right].
\end{equation*}
The above measure has the advantage of having less bias with respect to $\text{CAL}(\hat{\tau})$ than the plug-in estimator $n_{\mathcal{V}}^{-1}\sum_{i: w_i \in \mathcal{V}}\left[\hat{\gamma}_0(\hat{\tau},w_i) - \hat{\tau}(w_i)\right]^2$. 

\section{Simulation results}
\label{appendix::simResults}


\begin{table*}[htb]\small
\caption{Scenario 1 bias within bins of predictions for the calibrated and uncalibrated estimators. Each row shows the resulting bias for a given CATE estimator, and the Cal column indicates if it is calibrated or not. The columns are organized by sample size, and within each sample size, we show the results for the bias in the upper and lower deciles. Abbreviations: calibrated (cal), estimator (est), generalized additive models (GAM), generalized linear model (GLM), generalized linear model with lasso regularization (GLMnet), gradient boosted regression trees (GBRT), random forests (RF), multivariate adaptive regression splines (MARS).}
\label{tabApp:bwb1}
\vskip 0.15in
\centering
\begin{tabular}{ll|cc|cc|cc|cc}
\multicolumn{2}{l}{Sample Size}                                 & \multicolumn{2}{|c}{1000}                                                                                                                                  & \multicolumn{2}{|c}{2000}                                                                                                                                  & \multicolumn{2}{|c}{5000}                                                                                                                                  & \multicolumn{2}{|c}{10000}                                                                                                                                 \\ \hline
Cal & \begin{tabular}[c]{@{}l@{}}CATE \\ estimator\end{tabular} & \multicolumn{1}{|l}{\begin{tabular}[c]{@{}l@{}}Lower \\ Decile\end{tabular}} & \multicolumn{1}{l}{\begin{tabular}[c]{@{}l@{}}Upper \\ Decile\end{tabular}} & \multicolumn{1}{|l}{\begin{tabular}[c]{@{}l@{}}Lower \\ Decile\end{tabular}} & \multicolumn{1}{l}{\begin{tabular}[c]{@{}l@{}}Upper \\ Decile\end{tabular}} & \multicolumn{1}{|l}{\begin{tabular}[c]{@{}l@{}}Lower \\ Decile\end{tabular}} & \multicolumn{1}{l}{\begin{tabular}[c]{@{}l@{}}Upper \\ Decile\end{tabular}} & \multicolumn{1}{|l}{\begin{tabular}[c]{@{}l@{}}Lower \\ Decile\end{tabular}} & \multicolumn{1}{l}{\begin{tabular}[c]{@{}l@{}}Upper \\ Decile\end{tabular}} \\ \hline
yes & MARS    & -0.01 & -0.02 & 0.01 & -0.01 & 0    & -0.02 & 0    & -0.01 \\
no  & MARS    & -0.23 & 0.23  & -0.13 & 0.14  & -0.06 & 0.06  & -0.02 & 0.03  \\ \hline
yes & GAM     & -0.04 & 0.02  & -0.01 & 0.03  & 0    & 0.01  & 0    & 0     \\
no  & GAM     & -0.08 & 0.04  & -0.04 & 0.01  & -0.02 & 0     & -0.01 & -0.01 \\ \hline
yes & GLM     & -0.05 & 0.04  & -0.02 & 0.03  & -0.02 & 0.02  & -0.01 & 0.02  \\
no  & GLM     & -0.02 & 0.05  & 0.02  & 0.03  & 0.02  & 0.01  & 0.03  & 0.02  \\ \hline
yes & GLMnet  & -0.05 & 0.04  & -0.02 & 0.02  & -0.02 & 0.02  & -0.01 & 0.02  \\
no  & GLMnet  & 0     & 0.03  & 0.03  & 0.02  & 0.03  & 0.01  & 0.03  & 0.01  \\ \hline
yes & RF    & -0.06 & 0.03  & -0.01 & 0.02  & -0.01 & -0.01 & -0.01 & 0     \\
no  & RF   & -0.34 & 0.34  & -0.3  & 0.31  & -0.28 & 0.27  & -0.24 & 0.25  \\ \hline
yes & GBRT 2       & -0.03 & 0     & 0     & -0.01 & -0.01 & -0.01 & 0     & 0     \\
no  & GBRT 2      & -0.15 & 0.14  & -0.05 & 0.05  & -0.01 & -0.03 & 0.01  & -0.04 \\  \hline
yes & GBRT 5 & -0.01 & -0.03 & 0.03  & -0.06 & 0.03  & -0.06 & 0.03  & -0.05 \\
no  & GBRT 5 & -0.49 & 0.51  & -0.34 & 0.37  & -0.19 & 0.2   & -0.1  & 0.12  \\ \hline
yes & GBRT 8 & -0.02 & -0.03 & 0.02  & -0.06 & 0.05  & -0.09 & 0.05  & -0.09 \\
no  & GBRT 8 & -0.67 & 0.74  & -0.54 & 0.6   & -0.39 & 0.42  & -0.27 & 0.32  \\ \hline 
\end{tabular}
\end{table*}

\begin{table*}[htb]\small
\caption{Scenario 2 bias within bins of predictions for the calibrated and uncalibrated estimators. Each row shows the resulting bias for a given CATE estimator, and the Cal column indicates if it is calibrated or not. The columns are organized by sample size, and within each sample size, we show the results for the bias in the upper and lower deciles. Abbreviations: calibrated (cal), generalized linear model with lasso regularization (GLMnet), gradient boosted regression trees with GLMNet screening (GLMNet scr + GBRT).}
\label{tabApp:bwb2}
\vskip 0.15in
\centering
\begin{tabular}{ll|cc|cc|cc|cc}
\multicolumn{2}{l}{Sample Size} & \multicolumn{2}{|c}{1000} & \multicolumn{2}{|c}{2000} & \multicolumn{2}{|c}{5000} & \multicolumn{2}{|c}{10000} \\ \hline
Cal & \begin{tabular}[c]{@{}l@{}}CATE \\ estimator\end{tabular} & \multicolumn{1}{|l}{\begin{tabular}[c]{@{}l@{}}Lower \\ Decile\end{tabular}} & \multicolumn{1}{l}{\begin{tabular}[c]{@{}l@{}}Upper \\ Decile\end{tabular}} & \multicolumn{1}{|l}{\begin{tabular}[c]{@{}l@{}}Lower \\ Decile\end{tabular}} & \multicolumn{1}{l}{\begin{tabular}[c]{@{}l@{}}Upper \\ Decile\end{tabular}} & \multicolumn{1}{|l}{\begin{tabular}[c]{@{}l@{}}Lower \\ Decile\end{tabular}} & \multicolumn{1}{l}{\begin{tabular}[c]{@{}l@{}}Upper \\ Decile\end{tabular}} & \multicolumn{1}{|l}{\begin{tabular}[c]{@{}l@{}}Lower \\ Decile\end{tabular}} & \multicolumn{1}{l}{\begin{tabular}[c]{@{}l@{}}Upper \\ Decile\end{tabular}} \\ \hline
yes & GLMnet & -0.01 & 0.01 & -0.04 & -0.01 & -0.04 & -0.01 & -0.03 & -0.01 \\ \hline
no & GLMnet & -0.11 & -0.07 & -0.11 & -0.06 & -0.08 & -0.04 & -0.07 & -0.03 \\ \hline
yes & \begin{tabular}[c]{@{}l@{}}GLMnet scr \\ + GBRT \end{tabular} & -0.11 & -0.08 & -0.12 & -0.08 & -0.12 & -0.07 & -0.1 & -0.06 \\ \hline
no & \begin{tabular}[c]{@{}l@{}}GLMnet scr \\ + GBRT \end{tabular} & 0.09 & 0.03 & 0.05 & 0.01 & 0.04 & 0.01 & 0.03 & 0 \\ \hline
yes & \begin{tabular}[c]{@{}l@{}}random \\ forest\end{tabular} & -0.03 & -0.01 & -0.03 & -0.02 & -0.04 & -0.02 & -0.03 & -0.02 \\ \hline
no & \begin{tabular}[c]{@{}l@{}}random \\ forest\end{tabular} & -0.8 & -0.41 & -0.72 & -0.38 & -0.62 & -0.33 & -0.54 & -0.29 \\ \hline
\end{tabular}
\end{table*}
 
\begin{figure}
   \centering
\begin{subfigure}[H]{0.45\linewidth}
 
\includegraphics[width= \linewidth]{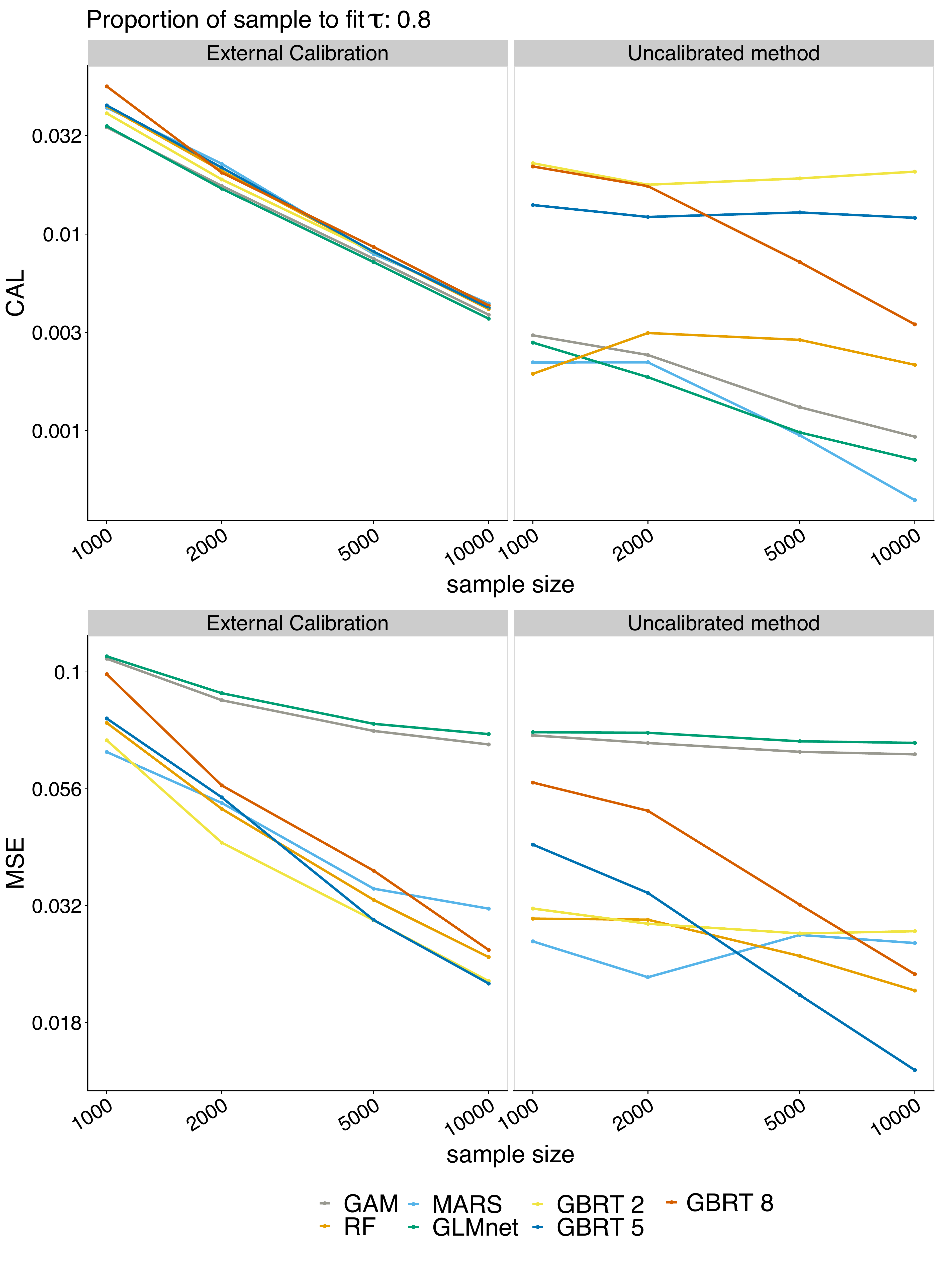}
\caption{Scenario 1 calibration measure and MSE simulation results for causal calibration approach with an external hold-out dataset. The top left and right panels show the  calibration measure and using conventional calibration and the uncalibrated estimator, respectively. Similarly, the bottom plots show MSE for the calibrated and uncalibrated estimators. Results for GLM and GBRT with depths of 3 and 6 are omitted because they were nearly identical to results shown for GLMnet and GBRT with other depths, respectively.}
\end{subfigure}\hspace{0.5cm}\begin{subfigure}[H]{0.45\linewidth}
\centering
\includegraphics[width= \linewidth]{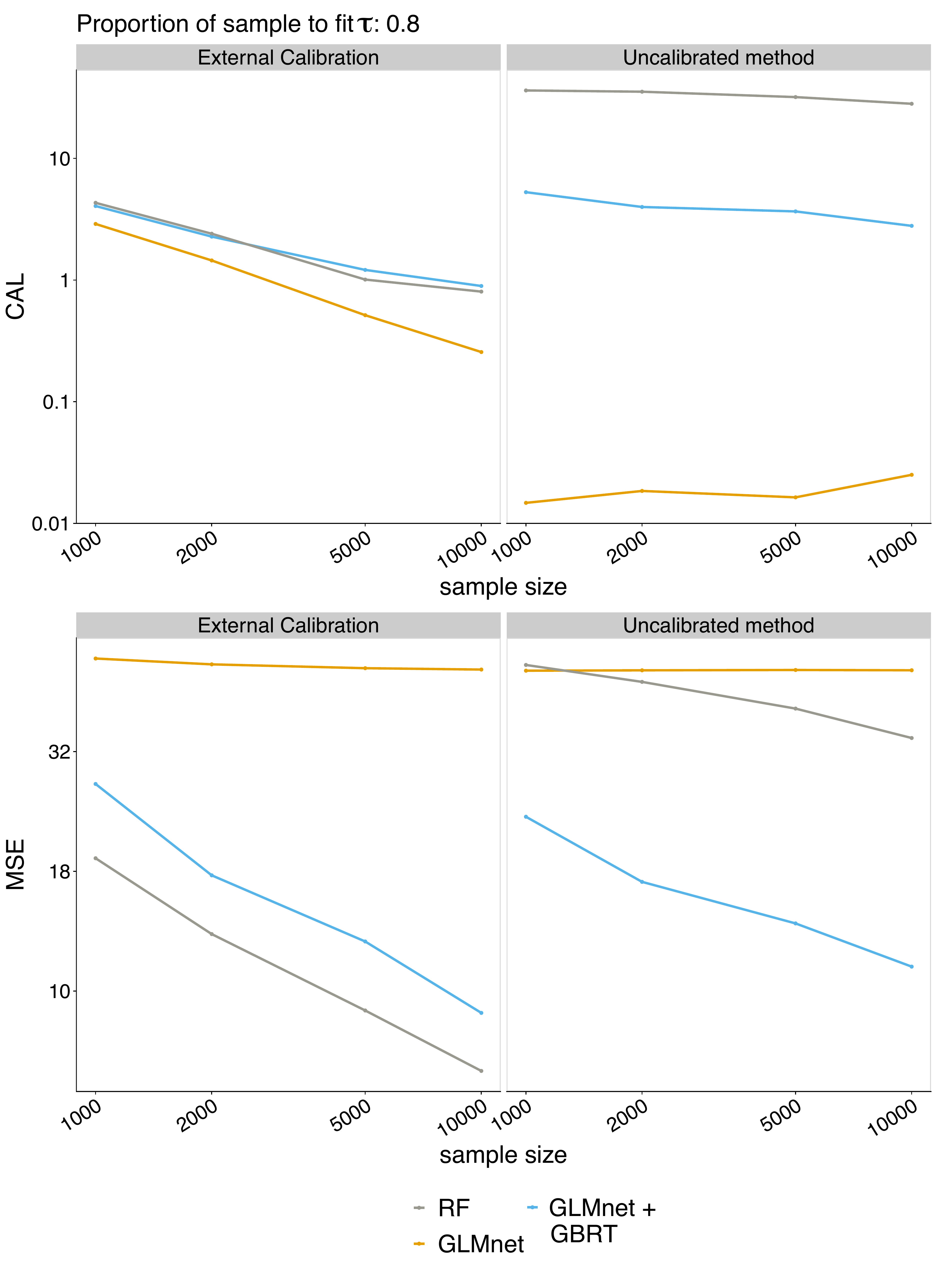}
\caption{Scenario 2 calibration measure and MSE simulation results for causal calibration approach with hold-out dataset. The top left and right panels show the calibration error using conventional calibration and the uncalibrated estimator, respectively. Similarly, the bottom plots show the MSE for the calibrated and uncalibrated estimators.}
\end{subfigure}
\caption{Causal isotonic calibration with a hold-out dataset external to the training dataset: Monte-Carlo estimates of calibration measure and MSE for calibrated vs uncalibrated predictors for Scenarios 1 and 2.}
\label{figApp:externalcal}
\end{figure}

\end{document}